\newif\ifarXiv         
\newif\ifjournal        

\journalfalse \arXivtrue  

\ifjournal 
\documentclass[final,onefignum,onetabnum]{siamart220329}

\usepackage{braket,amsfonts}

\usepackage{array}

\usepackage[caption=false]{subfig}
\usepackage{pgfplots}

\newsiamthm{claim}{Claim}
\newsiamremark{remark}{Remark}
\newsiamremark{hypothesis}{Hypothesis}
\crefname{hypothesis}{Hypothesis}{Hypotheses}

\usepackage{algorithmic}

\usepackage{graphicx,epstopdf}

\Crefname{ALC@unique}{Line}{Lines}

\usepackage{amsopn}

\usepackage{xspace}
\usepackage{bold-extra}
\usepackage[most]{tcolorbox}

\newcounter{example}
\colorlet{texcscolor}{blue!50!black}
\colorlet{texemcolor}{red!70!black}
\colorlet{texpreamble}{red!70!black}
\colorlet{codebackground}{black!25!white!25}

\lstdefinestyle{siamlatex}{%
  style=tcblatex,
  texcsstyle=*\color{texcscolor},
  texcsstyle=[2]\color{texemcolor},
  keywordstyle=[2]\color{texemcolor},
  moretexcs={cref,Cref,maketitle,mathcal,text,headers,email,url},
}

\tcbset{%
  colframe=black!75!white!75,
  coltitle=white,
  colback=codebackground, 
  colbacklower=white, 
  fonttitle=\bfseries,
  arc=0pt,outer arc=0pt,
  top=1pt,bottom=1pt,left=1mm,right=1mm,middle=1mm,boxsep=1mm,
  leftrule=0.3mm,rightrule=0.3mm,toprule=0.3mm,bottomrule=0.3mm,
  listing options={style=siamlatex}
}

\newtcblisting[use counter=example]{example}[2][]{%
  title={Example~\thetcbcounter: #2},#1}

\newtcbinputlisting[use counter=example]{\examplefile}[3][]{%
  title={Example~\thetcbcounter: #2},listing file={#3},#1}

\DeclareTotalTCBox{\code}{ v O{} }
{ 
  fontupper=\ttfamily\color{black},
  nobeforeafter,
  tcbox raise base,
  colback=codebackground,colframe=white,
  top=0pt,bottom=0pt,left=0mm,right=0mm,
  leftrule=0pt,rightrule=0pt,toprule=0mm,bottomrule=0mm,
  boxsep=0.5mm,
  #2}{#1}

\patchcmd\newpage{\vfil}{}{}{}
\theoremstyle{plain}
\newtheorem{thm}{Theorem}[section]
\newtheorem{conj}[thm]{Conjecture}
\newtheorem{Lem}[thm]{Lemma}
\newtheorem{lem}[thm]{Lemma}
\newtheorem{cor}[thm]{Corollary}
\newtheorem{prop}[thm]{Proposition}
\newtheorem{assumption}{Assumption}
\newtheorem{rem}[thm]{Remark}
\newtheorem{ex}[thm]{Example}
\numberwithin{equation}{section}

\newcommand{\balpha}{\boldsymbol{\alpha}}

\DeclareMathOperator*{\argmin}{arg\,min}
\newcommand{\rbracket}[1]{\left(#1\right)}      
\newcommand{\sbracket}[1]{\left[#1\right]}      
\newcommand{\cbracket}[1]{\left\{#1\right\}}      
\newcommand{\norm}[1]{\left\|#1\right\|}
\newcommand{\abs}[1]{\left|#1\right|}
\newcommand{\innerp}[1]{\langle{#1}\rangle}

\def\mL{\mathcal{L}}
\def\mE{\mathcal{E}}

\def\mN{\mathcal{N}}

\def\R{\mathbb{R}}

\newcommand{\bR}{\mathbf{R}}

\newcommand{\bA}{\mathbf{A}}
\newcommand{\bb}{\mathbf{b}}

\newcommand{\bv}{\mathbf{v}}
\newcommand{\ba}{\mathbf{a}}
\newcommand{\br}{\mathbf{r}}
\newcommand{\bh}{\mathbf{h}}
\newcommand{\bt}{\mathbf{t}}

\newcommand{\bu}{\mathbf{u}}

\newcommand{\bc}{\mathbf{c}}
\newcommand{\bw}{\mathbf{w}}
\newcommand{\bZ}{\mathbf{Z}}
\newcommand{\boldeta}{\boldsymbol{\eta}}
\newcommand{\boldlambda}{\boldsymbol{\lambda}}
\newcommand{\boldalpha}{\boldsymbol{\alpha}}
\newcommand{\boldxi}{\boldsymbol{\xi}}
\newcommand{\boldphi}{\boldsymbol{\varphi}}
\newcommand{\Log}{\text{Log}}

\usepackage[
backend=biber,
style=alphabetic,
isbn=false,
maxbibnames=9,
]{biblatex}
\DeclareSourcemap{
  \maps[datatype=bibtex, overwrite]{
    \map{
      \step[fieldset=address, null]
      \step[fieldset=editor, null]
      \step[fieldset=location, null]
    }
  }
}

\AtEveryBibitem{%
  	\clearfield{issn} 
  	\clearfield{doi} 
  	\clearfield{urldate}
	\clearfield{month}
	\clearfield{adress}
  	\ifentrytype{online}{}{
    \clearfield{url}
  }
}
\renewbibmacro*{url+urldate}{}
\renewbibmacro*{editor}{}
\renewbibmacro*{editors}{}
\renewbibmacro{in:}{}

\begin{tcbverbatimwrite}{tmp_\jobname_header.tex}
\title{Learning Memory Kernels in Generalized Langevin Equations\thanks{Submitted to the editors \today.
\funding{This work is supported in part by NSF DMS-2309378.}}}

\author{Quanjun Lang\thanks{Department of Mathematics, Duke University, Durham, NC 27708 USA \\(\email{quanjun.lang@duke.edu}, \email{jianfeng@math.duke.edu}).}
\and Jianfeng Lu\footnotemark[2]}

\headers{Learning Memory Kernels in Generalized Langevin Equations}{Quanjun Lang and Jianfeng Lu}
\end{tcbverbatimwrite}
\input{tmp_\jobname_header.tex}

\usepackage[mathlines]{lineno}
\linenumbers
\fi

\ifarXiv     
\documentclass[11pt]{article}  

\usepackage{amssymb}
\usepackage{amsthm}    
\usepackage{graphicx}          
\usepackage{amsmath}         
\usepackage{mathtools}
\mathtoolsset{showonlyrefs}

\usepackage{subcaption}
\usepackage{dsfont}
\usepackage{listings}
\usepackage{titlepic}
\usepackage{xcolor}
\usepackage{algorithm}
\usepackage{algorithmic}
\usepackage{todonotes}
\usepackage{array}
\usepackage{booktabs}
\usepackage{bm}
\usepackage{enumerate}
\usepackage{todonotes}

\usepackage{hyperref}
\usepackage{geometry}
\usepackage{tocloft}

\usepackage{enumitem}

\usepackage[
backend=biber,
style=alphabetic,
isbn=false,
maxbibnames=9,
]{biblatex}
\DeclareSourcemap{
  \maps[datatype=bibtex, overwrite]{
    \map{
      \step[fieldset=address, null]
      \step[fieldset=editor, null]
      \step[fieldset=location, null]
    }
  }
}

\AtEveryBibitem{%
  	\clearfield{issn} 
  	\clearfield{doi} 
  	\clearfield{urldate}
	\clearfield{month}
	\clearfield{adress}
  	\ifentrytype{online}{}{
    \clearfield{url}
  }
}
\renewbibmacro*{url+urldate}{}
\renewbibmacro*{editor}{}
\renewbibmacro*{editors}{}
\renewbibmacro{in:}{}

\usepackage{tikz}
\usetikzlibrary{decorations.pathreplacing,calc}

\hypersetup{
    colorlinks=true,
    linkcolor=blue,
    filecolor=magenta,      
    urlcolor=cyan,
    citecolor = blue,
}

\usepackage[width=\textwidth]{caption}

\theoremstyle{plain}
\newtheorem{thm}{Theorem}[section]
\newtheorem{theorem}{Theorem}[section]

\newtheorem{lemma}[thm]{Lemma}

\newtheorem{assumption}{Assumption}
\newtheorem{proposition}[thm]{Proposition}
\newtheorem{remark}[thm]{Remark}

\numberwithin{equation}{section}

\newcommand{\balpha}{\boldsymbol{\alpha}}

\DeclareMathOperator*{\argmin}{arg\,min}
\newcommand{\rbracket}[1]{\left(#1\right)}      
\newcommand{\sbracket}[1]{\left[#1\right]}      
\newcommand{\cbracket}[1]{\left\{#1\right\}}      
\newcommand{\norm}[1]{\left\|#1\right\|}
\newcommand{\abs}[1]{\left|#1\right|}
\newcommand{\innerp}[1]{\langle{#1}\rangle}

\def\mL{\mathcal{L}}
\def\mE{\mathcal{E}}

\def\mN{\mathcal{N}}

\def\R{\mathbb{R}}

\newcommand{\bR}{\mathbf{R}}

\newcommand{\bA}{\mathbf{A}}
\newcommand{\bb}{\mathbf{b}}

\newcommand{\bv}{\mathbf{v}}
\newcommand{\ba}{\mathbf{a}}
\newcommand{\br}{\mathbf{r}}
\newcommand{\bh}{\mathbf{h}}
\newcommand{\bt}{\mathbf{t}}

\newcommand{\bu}{\mathbf{u}}

\newcommand{\bc}{\mathbf{c}}
\newcommand{\bw}{\mathbf{w}}
\newcommand{\bZ}{\mathbf{Z}}
\newcommand{\boldeta}{\boldsymbol{\eta}}
\newcommand{\boldlambda}{\boldsymbol{\lambda}}
\newcommand{\boldalpha}{\boldsymbol{\alpha}}
\newcommand{\boldxi}{\boldsymbol{\xi}}
\newcommand{\boldphi}{\boldsymbol{\varphi}}
\newcommand{\Log}{\text{Log}}

\title{Learning Memory Kernels in Generalized Langevin Equations}
\author{Quanjun Lang, Jianfeng Lu}
\date{}

\fi

\newcommand{\reva}[1]{{\color{black}#1}} 
\newcommand{\revb}[1]{{\color{black}#1}}

\usepackage{caption}
\usepackage{subcaption}

\begin{document}
\maketitle

\ifjournal
\begin{tcbverbatimwrite}{tmp_\jobname_abstract.tex}
\fi

\begin{abstract}
We introduce a novel approach for learning memory kernels in Generalized Langevin Equations. This approach initially utilizes a regularized Prony method to estimate correlation functions from trajectory data, followed by regression over a Sobolev norm-based loss function with RKHS regularization. Our method guarantees improved performance within an exponentially weighted $L^2$ space, with the kernel estimation error controlled by the error in estimated correlation functions. We demonstrate the superiority of our estimator compared to other regression estimators that rely on $L^2$ loss functions and also an estimator derived from the inverse Laplace transform, using numerical examples that highlight its consistent advantage across various weight parameter selections. Additionally, we provide examples that include the application of force and drift terms in the equation.
\end{abstract}

\ifjournal
\end{tcbverbatimwrite}
\input{tmp_\jobname_abstract.tex}
\fi

\tableofcontents

\section{Introduction}\label{sec:intro}
The Generalized Langevin Equation (GLE) is a widely used model for coarse-grained particles, which was first proposed by Mori \cite{moriTransportCollectiveMotion1965} and Zwanzig \cite{zwanzigMemoryEffectsIrreversible1961}. In their formalism, many irrelevant degrees of freedom in the dynamics are projected, left with a memory term and a strongly correlated random noise. Such a coarse-graining simplification makes direct computation feasible. More examples of derivations can be found in \cite{chorinProblemReductionRenormalization2006, liCoarsegrainedMolecularDynamics2010}. The GLE is particularly useful for studying complex systems such as biomolecules \cite{gordonGeneralizedLangevinModels2009}, climate \cite{frankignoulStochasticClimateModels1977}, chemistry \cite{xieInitioGeneralizedLangevin2024}, physics \cite{luSemiclassicalGeneralizedLangevin2019}, and quantum dynamics simulation \cite{dammakQuantumThermalBath2009}, where
the interactions depend not only on the current state but also on the sequence of past states.

Because of the wide interest in modeling, many efforts have been devoted to learning the memory kernel. Still, it can be difficult to obtain, even assuming the full observation of the system. For example, the memory kernel is approximated by a rational function, thus reducing the model complexity
\cite{leiDatadrivenParameterizationGeneralized2016, groganDatadrivenMolecularModeling2020}; by introducing auxiliary variables based on Prony's estimation of the autocorrelation function \cite{bockiusModelReductionTechniques2021}; by data-driven model reduction methods \cite{linDatadrivenModelReduction2021, luComparisonContinuousDiscretetime2016}; and by an iterative algorithm which ensures accurate reproduction of the correlation functions \cite{jungIterativeReconstructionMemory2017}.

In estimating the memory kernel, the above methods empirically balance the computation cost and accuracy and justify the performance using numerical examples. To give a theoretical guarantee of the estimator, we address the ill-posed nature of this inverse problem, as the memory kernel satisfies the Volterra equation of the first kind. Solving these equations using various regularization methods has been studied in \cite{lammSurveyRegularizationMethods2000, lammFutureSequentialRegularizationMethods1995, lammNumericalSolutionFirstKind1997, lammRegularizedInversionFinitely1997}. Learning memory kernel by constructing loss function based on the Volterra of the first kind has been considered in \cite{russoMachineLearningMemory2022, kerrwinterDeepLearningApproach2023}. 

Our major contribution is an algorithm with a performance guarantee, depending on the estimation accuracy of the correlation functions, through constructing a Sobolev loss function. The correlation functions are estimated by a regularized Prony method, which empirically preserves the accuracy of the derivatives under the smoothness assumption. The performance is examined in an $L^2(\rho)$ space, where $\rho$ is an exponential decaying measure representing our region of interest in our estimation. Using the Prony method to estimate correlation functions allows for the derivation of the memory kernel via a direct inverse Laplace transform. However, this estimation approach does not offer a guaranteed level of performance and is primarily used as a benchmark for comparison. We also compare the performance of the proposed estimator with other regression estimators based on $L^2$ loss functions, showing that the proposed estimator has better performance across different scales of $\rho$. In the end, we give examples with the presence of force term $F$ and a drift term in the equation, demonstrating the necessity of using ensemble trajectory data when the solution is not stationary. 

The paper is organized as follows. In Section \ref{sec_preliminary}, we give preliminaries of the problem statement and introduce the notations. We propose the main algorithm in Section \ref{sec_method}. The identifiability and convergence of our estimator are proved in Section \ref{sec_id}, and the performance is demonstrated by numerical examples in Section~\ref{sec_numerics}. We conclude in Section \ref{sec_conclusion}.

\section{Preliminary}\label{sec_preliminary}
\subsection{Generalized Langevin equations (GLE)}
In this paper, we restrict the discussion to one dimension and leave the generalization to high dimensions for future work. We consider the following GLE.
\begin{equation}\label{eq_main_GLE}
	mv'(t) = F(v(t)) -\int_0^{t}\gamma(t-s)v(s)ds + R(t),
\end{equation}

where $v:[0, \infty) \rightarrow \R$ is the velocity of a macroparticle, $m$ represents its mass which is assumed to be 1 in the rest of our discussion, and $F$ represents \reva{an external force term}. The memory kernel $\gamma:[0,\infty) \rightarrow \R$ is assumed to be continuous and integrable, and $R(t)$ is a stationary Gaussian process with zero mean, satisfying the fluctuation-dissipation theorem, so that
\begin{equation}\label{eq_fluctuation_dissipation}
	\innerp{R(t)R(s)} = \frac{1}{\beta}\gamma(t - s),
\end{equation}
where $\beta$ is the inverse temperature. We aim to estimate the memory kernel $\gamma$ from the discrete observations of trajectory data $v$.

Multiply $v(0)$ to the \eqref{eq_main_GLE} and take expectations, we obtain 
\begin{equation}\label{eq_main_volterra}
	g(t) = \int_0^t \gamma(t - s) h(s) ds = \int_{0}^t \gamma(s)h(t-s)ds,
\end{equation}
where 
\begin{align}\label{eq_phi_h}
	 h(t) := \innerp{v(t)v(0)}, \ \ \ \varphi(t) := \innerp{F(v(t))v(0)}, 
\end{align}

so that
\begin{align}\label{eq_g}
	g(t) := \innerp{v'(t)v(0)} -\innerp{F(v(t))v(0)} = - h'(t) + \varphi(t).
\end{align}
and $\innerp{R(t)v(0)}$ is assumed to be 0 as in \cite{chenComputationMemoryFunctions2014}. In the case that the corresponding solution $v$ is a stationary ergodic Gaussian process with zero means, $h(t)$ defines the autocorrelation function of $v$, namely $h(t - s) = \innerp{v(t)v(s)}$\reva{, and we call $g$ the force correlation function.} Throughout the following discussion, we will assume that $g$, $\varphi$, and $h$ are smooth. \reva{For the special case \( F = 0 \), we have \( g(t) = -h'(t) \). When \( F(v) = -\mu v \), this becomes \( g(t) = -h'(t) + \mu h(t) \), recovering the standard GLE with linear friction \(\mu\) and no external potential. 

Although the commonly used GLE formulation applies the force term to the position variable, here we intentionally consider a simplified first-order model where the force acts on the velocity. This setting allows us to focus on the identifiability of the memory kernel while also illustrating cases where the dynamics are non-stationary, such as when the force is nonlinear in $v$. Nevertheless, we point out that velocity-dependent potentials \cite{pucacco2004integrable} appear in the Lorentz force \cite{de2017spin} and Nuclear physics \cite{razavy1962analytical}.
}

\begin{remark}
	If $v$ is not stationary, we cannot approximate $h$ by the autocorrelation. Such a case might appear with the presence of a nonlinear force $F$ and an extra drift term. Therefore, ensemble trajectory data is necessary to construct the correlation functions $h$ and $g$. The details will be illustrated in Section \ref{sec_numerics}. 
\end{remark}

\subsection{Volterra equations}
Note that the equation \eqref{eq_main_volterra} is a convolutional Volterra equation of the first kind, which is known to be ill-posed, meaning a small error of $h$ and $g$ will cause an arbitrarily large error in the estimated $\gamma$. In the following discussion, we use the notation 
\begin{equation}
	(\gamma*h)(t) := \int_0^t \gamma(t - s)h(s)ds = \int_0^t \gamma(s)h(t - s)ds.
\end{equation}
Taking the derivative with $t$, the first kind Volterra equation \eqref{eq_main_volterra} becomes a second kind, namely
\begin{equation}\label{eq_Volterra_second_kind}
	g'(t) = \int_0^t \gamma(t - s) h'(s) ds + h(0)\gamma(t) = (\gamma*h')(t) + h(0)\gamma(t).
\end{equation}
For the case that $h(0) \neq 0$, this equation is well-posed, therefore shaping our thoughts of constructing the loss function.
The ill-posedness of the equation \eqref{eq_main_volterra} is because of the zero spectrum in the convolution operator, resulting in an unbounded inverse. However, when $h(0) \neq 0$, the spectrum of the operator in equation \eqref{eq_Volterra_second_kind} has a lower bound $h(0)$, making its inverse bounded. Many efforts have been devoted to providing regularization techniques for solving the Volterra equation of the first kind \cite{lammSurveyRegularizationMethods2000, lammFutureSequentialRegularizationMethods1995, lammNumericalSolutionFirstKind1997, lammRegularizedInversionFinitely1997}. 


\subsection{Laplace transform}
The Laplace transform of $f\in L^2(\R^+)$ is defined as
\begin{equation}
	\mL[f] = \widehat f (z) = \int_0^\infty f(t) e^{-zt}dt.
\end{equation}
The integral is a proper Lebesgue integral for complex domain parameters $z = \omega + i\tau$ such that $\omega \geq 0$. A  stronger assumption is that $f$ has exponential decay, 
\begin{equation}
	\abs{f(t)} \leq C e^{-\sigma t}
\end{equation}
for some $C, \sigma > 0$. In later discussions, we assume that the memory kernel $\gamma$ and the correlation function $h$, $\varphi$, and $g$ have exponential decay.

\begin{lemma}[Plancherel Theorem for Laplace transform]\label{lem_Plancherel}
Suppose $f$ has exponential decay. Then the following equality holds for $\omega \geq 0$.
\begin{equation}\label{eq_Plancherel_Laplace_exponential_decay}
	\frac{1}{2\pi}\int_{-\infty}^{+\infty} \abs{\widehat f(w + i\tau)}^2 d \tau = \int_0^\infty e^{-2wt}\abs{f(t)}^2 dt.
\end{equation}
\end{lemma}
\begin{proof}
From the definition of Laplace transform,
\begin{equation}
	\widehat f (\omega + i\tau ) = \int_0^\infty e^{-it\tau}e^{-\omega t}f(t)dt. 
\end{equation}
Notice that $\widehat f(\omega + i\cdot)$ is the Fourier transform of $e^{-w t}f(t)$ by extending $f = 0$ for $t \leq 0$. Using the assumption that $f$ has exponential decay, the extended function $e^{-w t}f(t) \in L^1(\R) \cap L^2(\R)$ for $\omega \geq 0$, hence $\widehat f(\omega + i\cdot) \in L^2(\R)$. Then the Plancherel theorem of Fourier transform provides an isometry of $f\in L^2(\rho)$ to its Laplace transform in $L^2(R)$ with a scaling constant $\frac{1}{2\pi}$, where the measure $\rho$ has density $e^{-2\omega t}$.
\end{proof}

The Lemma \ref{lem_Plancherel} will be applied to derive the coercivity constant for the loss function defined later in Section \ref{sec_method}. Moreover, the Laplace transform can convert the convolution into multiplication in the frequency domain.
\begin{lemma}[Laplace transform of convolution]\label{lem_lap_conv}
	Suppose $f_1, f_2$ both have exponential decay, then
\begin{equation}
	\mL\left[\int_0^t f_1(t-s) f_2(s)ds\right] = \widehat{f_1}(z) \cdot \widehat f_2(z).
\end{equation}
\end{lemma}
The inverse Laplace transform is given by the Bromwich integral,
\begin{equation}
	f(t) = \mL^{-1}\left[{\widehat f}(t)\right] = \frac{1}{2\pi i}\lim_{T \to \infty} \int_{\omega - iT}^{\omega + iT}e^{zt} \widehat {f}(z) dz
\end{equation}
where $\omega$ is a real number such that $\omega$ is on the right of all the poles in $\widehat {f}$. In particular, if $f$ has exponential decay, we could take $\omega \geq 0$. See \cite{kuhlmanReviewInverseLaplace2013} for a review of numerical methods for inverse Laplace transform. Among those algorithms, the method of accelerated Fourier series method is widely used \cite{dehoogImprovedMethodNumerical1982}. 
 
Laplace transform is particularly useful when solving the Volterra equations thanks to Lemma \ref{lem_lap_conv}. Taking the Laplace transform on both sides of the Volterra equation \eqref{eq_main_volterra}, we have 
\begin{equation}
	\widehat{g}(z) = \widehat{\gamma}(z)\widehat{h}(z), \ \text{and } \gamma = \mL^{-1}\left[\dfrac{\widehat{g}}{\widehat{h}}\right]
\end{equation}
However, the numerical inverse Laplace transform is severely unstable. See, for example, \cite{epsteinBadTruthLaplace2008} for a detailed discussion on the ill-posedness of inverse Laplace transform. We avoid such ill-posedness by explicitly computing the inverse Laplace transform of exponential functions, and also by avoid the inverses Laplace transform by minimizing a well-posed loss function.


%
%
%

\section{Proposed Method}\label{sec_method}
We present the main method in this section. The data is assumed to be noisy and discrete observations of the stationary solution $v$ of the GLE. We first estimate the correlation functions $h$, $\varphi$, and $g$ using regularized Prony methods. Then we apply regression based on a quadratic loss derived from a weighted Sobolev norm. We also derive an estimator for comparison purposes, using the explicit inverse Laplace transform based on the Prony estimation of the correlation functions.
\subsection{Observation data}
Suppose we have the following discrete and noisy observation data
\begin{equation}\label{eq_noisy_traj}
	\bv_l = v(\bt_l) + \boldxi_l
\end{equation}
where $\bt_l = l\Delta t$, $l = 1, \dots, L$ and $\boldxi_l$ are i.i.d. Gaussian noise with mean 0 and standard deviation $\sigma_{obs}$, independent of the initial condition $v(0)$. \reva{This setup corresponds to the first-order GLE model, where velocity is assumed to be directly observed. While in many experimental settings, position is more commonly measured, velocity measurements are also feasible. For instance, Doppler-based techniques, such as Laser Doppler Velocimetry (LDV) \cite{tabatabai2013novel} and Doppler radar \cite{atlas1973doppler}, provide direct access to velocity data and fall within our observation model. If only positional observations are available, we could estimate the velocity using finite difference or interpolation and subsequently use it to compute the correlation functions.} 

Under the assumption that the solution $v$ is ergodic, the autocorrelation functions $h$ defined as in \eqref{eq_phi_h} equals to  the temporal average of $v$,
\begin{equation}\label{eq_cts_acf}
	h(\tau) = \lim_{T \to \infty}\frac{1}{T}\int_0^T v(t)v(t+\tau)dt.
\end{equation}
The above integral can be approximated by a discrete sum, 
\begin{equation}\label{eq_discrete_acf}
	\bh_n = \frac{1}{L-n}\sum_{l = 1}^{L-n} \bv_{l}\bv_{l + n},
\end{equation}
and taking into account \eqref{eq_noisy_traj},
\begin{equation}
	\bh_n = \frac{1}{L-n}\rbracket{\sum_{l = 1}^{L - n} v(\bt_l)v(\bt_{l+n}) + \sum_{l = n}^{L - n} v(\bt_l)\bold{\xi}_{l + n} + \sum_{l = 1}^{L - n} \bold{\xi}_{l}v(\bt_{l+n}) + \sum_{l = 1}^{L-n} \bold{\xi}_{l}\bold{\xi}_{l + n}}.
\end{equation}
Due to the independence of $\boldeta_n$, the second and third term converges to $0$ as $M$ approaches infinity because of the central limit theorem, and the last term converges to 0 except for the case $n = 0$. As a result, the $\bh_0$ is an estimation  of $h(0) + \sigma_{obs}$. Moreover, when the memory kernel is integrable and continuous, the autocorrelation function $h$ is differentiable with 
	$h'(0) = 0$, see for example \cite{baczewskiNumericalIntegrationExtended2013}. 
In our subsequent discussion, we impose this condition and eliminate the estimation $\bh_0$, focusing solely on the data $\{\bh_n\}_{n = 1}^N$.

\subsection{Regularized Prony method}\label{subsec_reg_prony}
The well-known Prony method assumes the target function $h$ is a summation of exponential functions \cite{hauerInitialResultsProny1990}. It transfers the nonlinear parameter estimation into two linear regressions and a root-finding problem. For completeness, we briefly discuss the Prony method and introduce our regularization methods according to our assumptions. 

Given discrete estimations $\{\bh_n\}_{n = 1}^N$, we aim to find $(\bw_k, \br_k)_{k = 1}^{p'}$, so that 

\begin{equation}
	\reva{\min_{\{\bold{w}_k\}} \left\| \bh_n - \sum_{k=1}^{p{\prime}} \bold{w}_k \br_k^n \right\|^2,}
\end{equation}
where $\bw_k$ are the amplitudes, $\br_k$ are the exponential of complex frequencies. Choosing the number of frequencies ${p'}$ and appropriate observation length $N$ is another problem of interest. The choice of $N$ reflects the precision of estimated $\bh_n$, as larger values of $n$ correspond to smaller sample sizes $L-n$, leading to increased estimation errors. The balance between increasing the number of estimation grids and maintaining estimation accuracy is achieved through careful consideration of the parameter $N$. We refer to \cite{carriereHighResolutionRadar1992} for a singular value decomposition approach in selecting the optimal parameter $N$ and $p'$, and \cite{bockiusModelReductionTechniques2021} for numerical comparisons. The classical case ($N = 2{p'}$) obtains an exact fit between the sampled data and the exponentials if the regression matrices defined later are not singular. However, in practical cases, we often require $N > 2{p'}$ and solve linear systems using least squares. 


Firstly we solve the following equation for coefficients $\ba = (\ba_1, \dots, \ba_{p'})^\top$. 
\begin{equation}\label{eq_prony_poly_coef_a}
	\begin{bmatrix}
		\bh_{p'} & \bh_{{p'}-1} & \cdots & \bh_1\\
		\bh_{{p'}+1} & \bh_{{p'}} & \cdots & \bh_{2}\\
		\vdots & \vdots & \ddots & \vdots \\
		\bh_{N-1} & \bh_{N-2} & \cdots & \bh_{N - {p'}}
	\end{bmatrix}
	\begin{bmatrix}
		\ba_1\\ \vdots \\ \ba_{p'}	
	\end{bmatrix}
	=
	-
	\begin{bmatrix}
		\bh_{{p'}+1}\\ \vdots \\ \bh_{N}
	\end{bmatrix}
\end{equation}
Given the coefficients $\bold{a}$, we construct the characteristic polynomial, 
\begin{equation}
	\phi(z) = z^{p'} + \ba_1 z^{{p'}-1} + \cdots \ba_{{p'}-1}z + \ba_{p'}.
\end{equation} 
The second step is finding the roots of $\phi(z)$, denoted by $\{\br_1, \dots, \br_{p'}\}$. 
The process of root finding is converted into a generalized eigenvalue problem known as the matrix pencil method \cite{huaMatrixPencilMethod1990} to enhance stability. Yet, this root-finding is still very sensitive to small changes in coefficients, which can lead to disproportionately large errors in the estimation of roots. See, for example, the Wilkinson's polynomial \cite{wilkinsonEvaluationZerosIllconditioned1959}. However, we emphasize that instead of precisely determining the parameters $(\bw_k, \br_k)$, our primary objective is to achieve a nonparametric estimation of the function $h$ minimizing the $L^2(\rho)$ error, with a measure $\rho$ to be defined later. Such a nonparametric goal justifies our following regularization steps. To incorporate with the assumption of exponential decay in $h$, we adjust the root $\br_k$ when its absolute value exceeds the threshold $\sigma$ of the exponential decaying assumption, namely
\begin{equation}\label{eq_prony_root_aug}
	\widetilde{\br_k} = \frac{\br_k}{e^{\sigma}\abs{\br_k}}, \text{ if } \abs{\br_k} \geq e^{-\sigma}, \widetilde{\br_k} = \br_k \text{ otherwise }.
\end{equation}
\reva{
We then define the exponential modes
\begin{equation}
	\boldlambda_k = \frac{\text{Log}(\widetilde{\br_k})}{\Delta t}.
\end{equation}
\begin{remark}\label{rmk_prnoy_root_aug}
The above formula applies when \( \widetilde{\br}_k \notin \mathbb{R}^- \). If \( \widetilde{\br}_k \in \mathbb{R}^- \), it lies on the branch cut of the complex logarithm, and \( \Log(\widetilde{\br}_k) \) becomes multivalued. To solve this, we augment the parameters $\boldlambda_k$ to include the two values on both sides of the cutline \cite{bockiusModelReductionTechniques2021}, $\log\abs{\br_k} + \pi i$ and $\log\abs{\br_k} - \pi i$. 	
\end{remark}
}
After rearranging, we have $\{{\boldlambda}_k\}_{k = 1}^{p}$, where $p \geq p'$ because the augmentation. In the final step of our process, we focus on determining the coefficients $\bw$. The original Prony solves the linear system denoted by $\bold{Z}\bw = \bold{h}$ (Note that we dropped the estimation of $\bh_0$),
\begin{equation}
	\begin{bmatrix}
		e^{{\boldlambda_1}\Delta t} & \dots & e^{{\boldlambda_{p}}\Delta t}\\
		\vdots & \ddots & \vdots \\
		e^{{\boldlambda_1}N\Delta t} & \cdots & e^{{\boldlambda_{p}}N\Delta t}
	\end{bmatrix}
	\begin{bmatrix}
		\bw_1\\ \vdots \\ \bw_{p}
	\end{bmatrix}
	=
	\begin{bmatrix}
		\bh_1 \\ \vdots \\ \bh_{N}
	\end{bmatrix}
\end{equation}
Given that the number of data points \( N \) exceeds \( p \), we employ a least squares solution with constraints based on $h'(0) = 0$, 
\begin{equation}\label{eq_dh_est_0}
	\sum_{k = 1}^{p} \bw_k {\boldlambda}_k = 0,
\end{equation}
and RKHS regularization \cite{luDataAdaptiveRKHS2022}, which provides inherent smoothness constraints on $h$ determined by the matrix $\bold{Z}$ itself. The method we used to determine $\bw$ is concluded as 
\begin{equation}\label{eq_prony_weight_RKHS}
	\bw = \argmin_{\sum_{k = 1}^{p} \bw_k {\boldlambda}_k = 0}\norm{\bold{Z}\bw - \bh}^2 + \lambda\norm{(\bold{Z}\bold{Z}^\top)^{\dag}\bw}^2
\end{equation}
where $(\bZ^\top\bZ)^{\dag}$ represents the pseudo inverse of $\bZ^\top\bZ$, and the last term in the above equation represents the RKHS norm derived from $\bZ^\top \bZ$ and the optimal parameter $\lambda$ is selected by L-curve method. See \cite{luDataAdaptiveRKHS2022} for more details.

Overall, our regularized Prony method distinguishes from the classic approach by omitting the biased estimation of \( \bh_0 \), segregating multi-valued Prony modes, constraining decay speed and \( h'(0) \), and adding regularization to the coefficients. \reva{ Now the target function $h$ is approximated by
\begin{equation}\label{eq_prony_h}
	\widetilde{h}(t) = \sum_{k = 1}^{p} \bw_k e^{\boldlambda_k t}.
\end{equation}}
The method is concluded in the Algorithm \ref{alg_reg_Prony}.

\begin{algorithm}
\caption{Regularized Prony method}\label{alg_reg_Prony}
    \hspace*{\algorithmicindent} \textbf{Input:} Trajectory $\{\bv_l\}_{l = 1}^L$ observed on discrete time grid $\{\bt_l\}_{l = 1}^L$.\\
    \hspace*{\algorithmicindent} \textbf{Output:} Estimated auto correlation function $\widetilde{h}.$
\begin{algorithmic}[1]
\STATE{Estimate  $\bh$ on discrete time grids using \eqref{eq_discrete_acf} and drop $\bh_0$.}
\STATE{Choose $N$ and $p'$ and estimate the polynomial coefficients $\ba$ using \eqref{eq_prony_poly_coef_a}.}
\STATE{Find the roots $\br$ of the characteristic polynomial $\phi$ and regularize using \eqref{eq_prony_root_aug}.}
\STATE{Augment the logarithm $\{\boldlambda_k\}_{k = 1}^p$ of $\br$ using Remark \ref{rmk_prnoy_root_aug}.}
\STATE{Determine the weight $\bw$ using least square with RKHS regularization \eqref{eq_prony_weight_RKHS}, so that $\widetilde{h} = \sum_{k = 1}^p \bw_k e^{\boldlambda_k t}$.}
\end{algorithmic}
\end{algorithm}

%

\subsection{The force term $F$}

The process for estimating \( g \) follows a similar pattern, but with a distinction based on  \( F \). We divide the scenario into two cases: when \( F \) equals 0 and when it does not.

\subsubsection{The case $F = 0$} 
From \eqref{eq_g}, we have $g(t) = -h'(t)$. Take Laplace transform, we have $\widehat{g}(z) = h(0) - z\widehat{h}(z)$. Therefore the following relationship holds for the Laplace transforms of $\gamma$ and $h$. 
\begin{equation}\label{eq_gamma_h_relation}
	\widehat{\gamma}(z) = \frac{\widehat{g}(z)}{\widehat{h}(z)} = \frac{h(0) - z\widehat{h}(z)}{\widehat{h}(z)}, \ \ \widehat{h}(z) = \frac{h(0)}{z + \widehat{\gamma}(z)}
\end{equation}
With the Prony approximation $\widetilde{h}$ derived from above, we firstly have a straightforward estimation of $g$,
\begin{equation}\label{eq_widetilde_dh}
	\widetilde{g}(t) = -\widetilde{h}'(t) = -\sum_{k = 1}^{p} {\bw_k\boldlambda_k} e^{\boldlambda_k t}
\end{equation}

With $\widetilde{g}$ ready for later use, here we briefly introduce a direct estimation of $\gamma$ based on explicit inverse Laplace transform. Note that the Laplace transform of $\widetilde{h}$ and $\widetilde{g}$ are given by 
\begin{equation}
	\mL[\widetilde{h}] = \sum_{k = 1}^p \frac{\bw_k}{z - \boldlambda_k}, \ \mL[\widetilde{g}] = -\sum_{k = 1}^p \frac{\bw_k\boldlambda_k}{z - \boldlambda_k}, 
\end{equation}
Using \eqref{eq_gamma_h_relation}, we can derive an estimator $\theta_{L}$ of the memory kernel, whose Laplace transform is given by
\begin{equation}\label{eq_theta_prony_inv}
	\widehat{\theta_{L}}(z) = \frac{ \mL[-\widetilde{h'}]}{ \mL[\widetilde{h}]} = \frac{-\sum_{k = 1}^{p}\bw_k\boldlambda_k\prod_{j \neq k}{(z - \boldlambda_j)}}
{\sum_{k = 1}^{p}\bw_k\prod_{j \neq k}{(z - \boldlambda_j)}}
\end{equation}
Because of the constraints \eqref{eq_dh_est_0} derived from $\widetilde{h}'(0) = 0$, the coefficients of the $(p-1)^{th}$ order term in the numerator is 0, therefore we can use fraction decomposition and explicit inverse Laplace transform,
\begin{equation}\label{eq_theta_prony_frac_decomp}
	\widehat{\theta_{L}}(z) = \sum_{k = 1}^{p-1}\frac{\bu_k}{z - \boldeta_k}, \ \ \theta_{L}(t) = \sum_{k =1}^{p-1}\bu_k e^{-\boldeta_k t}
\end{equation}


\begin{remark}\label{rmk_gamma_h}
Suppose the true memory kernel $\gamma$ is given by $\gamma(t) = \sum_{k = 1}^p u_k e^{\eta_k t}$, then we have from \eqref{eq_gamma_h_relation},
\begin{equation}
	\widehat{h}(z) = \frac{h(0)}{z + \sum_{k = 1}^p\frac{u_k}{z - \eta_k}} = \frac{h(0) \prod_{k = 1}^p(z- \eta_k)}{z\prod_{k = 1}^p(z- \eta_k) + \sum_{k = 1}^p u_k\prod_{j \neq k}(z - \eta_j)}.
\end{equation}
Notice that the numerator has degree $k$ and the denominator has degree $k+1$. Then using fractional decomposition, we have 
\begin{equation}
	\widehat{h}(s) = \sum_{k = 1}^{p+1} \frac{w_k}{s - \lambda_k}.
\end{equation}
with the parameters derived from the above equality. Therefore we obtain a relationship between a Prony-like memory kernel and its autocorrelation function $h$. This relationship will be applied in the numerical examples in Section \ref{sec_numerics}.

%
%
\end{remark}

\subsubsection{The case $F$ is not 0}
Note that in this case, $g = -h' + \varphi$. Firstly we use the $\widetilde{h}'$ from \eqref{eq_widetilde_dh}. 
Since $\varphi$ is given by the ergodic assumption on $v$, 
\begin{equation}
	\varphi(\tau) = \lim_{T \to \infty}\int_0^T v(t)F(v(t+\tau))dt,
\end{equation}
we use the Riemann sum to approximate the integral,
\begin{equation}
	\boldphi_n = \frac{1}{L-n}\sum_{l = 1}^{L - n} \bv_{l}F(\bv_{l+n}),
\end{equation}
and then use the regularized Prony method as concluded in the Algorithm \ref{alg_reg_Prony}. The result is denoted by
\begin{equation}
	\widetilde{\varphi}(t) = \sum_{k = 1}^{q}\bw_k'e^{\boldlambda_k't}	
\end{equation}
In conclusion, we have 
\begin{equation}\label{eq_prony_g}
	\widetilde{g}(t) = \sum_{k = 1}^{p} {\bw_k\boldlambda_k}e^{{\boldlambda_k}t} + \sum_{k = 1}^{q} {\bw_k'}e^{{\boldlambda_k'}t}
\end{equation}

%
%
%
%
%
%

Again we provide a brief discussion about estimating the memory kernel using explicit inverse Laplace transform. By equation \eqref{eq_gamma_h_relation},
\begin{equation}
	\widehat{{\theta_{L}}}(z)  = \frac{\mL[-\widetilde{h'}] + \mL[\widetilde{\varphi}]}{\mL[\widetilde{h}]} 
	= 
	\sum_{k = 1}^{p-1}\frac{\bu_k}{z - \boldeta_k} + \frac{\sbracket{\sum_{k = 1}^{q}\bw_k'\prod_{j \neq k}{(z - \boldlambda_j')}}\prod_{i=1}^p{(z - \boldlambda_j)}}
{\sbracket{\sum_{k = 1}^{p}\bw_k\prod_{j \neq k}{(z - \boldlambda_j)}}\prod_{i=1}^q{(z - \boldlambda_j')}}
\end{equation}
where the first term on the left-hand side is derived from equation \eqref{eq_theta_prony_frac_decomp}. 
Recall that $\varphi(0) = \left<F(v(0)), v(0)\right>$, we cannot have $\varphi(0) = 0$ as the case for $h'(0)$. Hence in the second term, both the numerator and the denominator have degree $p + q - 1$, and its fractional decomposition has a constant term $C = \frac{\sum_{k = 1}^q \bw_k'}{\sum_{k = 1}^p \bw_k}$, leading to $\delta_0(t)$,  a Dirac function at 0 in the estimation $\theta_L$, which is approximated by a mollified Gaussian density function in practice. The final result after partial fraction decomposition and explicit inverse Laplace transform is
\begin{equation}\label{eq_theta_prony_frac_decomp_F_not_0}
	\widehat{\theta_{L}}(z) = \sum_{k = 1}^{q'}\frac{\bu_k}{z - \boldeta_k} + C, \ \ \theta_{L}(t) = \sum_{k = 1}^{q'} \bu_ke^{\boldeta_k t} + C\delta_0(t)
\end{equation}

\subsection{Sobolev norm loss function}
Although we have the Prony-like estimator $\theta_L$, it is hard to give an analysis of its performance. Instead, with the 
given estimation of $\widetilde{h}$ and $\widetilde{g}$, we use least square to learn the memory kernel with the following loss function, 
\begin{equation}\label{eq_loss_Sobolev_norm}
	\mE(\theta) = \norm{\widetilde{g} - (\theta*\widetilde{h})}^2_{H^1_{\boldalpha}(\rho)}
\end{equation}
where $\theta$ is a candidate kernel function, $\rho$ denotes the measure on $\R^{+}$ with density $e^{-2\omega t}$ and $\norm{f}^2_{H^1_{\boldalpha}(\rho)}$ represents the squared Sobolev norm with scale parameter $\boldalpha = (\alpha_1, \alpha_2)$ which  satisfies $\alpha_1, \alpha_2 > 0$ and $\alpha_1 + \alpha_2 = 1$, so that 
\begin{equation}
\norm{f}_{H^{1}_{\boldalpha}(\rho)}^2  = \int_0^{+\infty} \rbracket{\alpha_1\abs{f(t)}^2 + \alpha_2\abs{f'(t)}^2} d\rho(t).
\end{equation}
The Sobolev norm provides coercivity with respect to the $L^2(\rho)$ norm. The details will be introduced in Section \ref{sec_id}. The measure $\rho$ can be considered as a weight of estimation. A large value of $\omega$ indicates our focus on the value of the memory kernel near the origin and vice versa. 
Using the basis $\cbracket{\psi_k}_{k = 1}^K$, 
the least square problem becomes estimating the coefficients $\bc = (\bc_1, \dots, \bc_K)^\top$, so chat
\begin{equation}
	\bc = \argmin_{\bc \in \R^K}\bc^\top \bA \bc - 2\bb^\top \bc.
\end{equation}
where $\bA \in \R^{K\times K}, \bb \in \R^{K}$, and 
\begin{equation}\label{eq_lsq_A_b}
	\bA_{ij} = \innerp{\psi_i*\widetilde{h}, \psi_j*\widetilde{h}}_{H^1_{\boldalpha}(\rho)}, \text{  }\bb_{i} = \innerp{\psi_i*\widetilde{h}, \widetilde{g}}_{H^1_{\boldalpha}(\rho)}
\end{equation}
We again use RKHS regularization \cite{luDataAdaptiveRKHS2022} as in Section \ref{subsec_reg_prony}, so that 
\begin{equation}\label{eq_lsq_RKHS_c}
	\bc = \argmin_{\bc \in \R^K}\bc^\top \bA \bc - 2\bb^\top \bc + \lambda \bc^\top \bA^\dag \bc
\end{equation}
Finally, the estimator of the memory kernel is given by $\theta = \sum_{k = 1}^K \bc_k \psi_k$. 
Our algorithm is concluded in Algorithm \ref{alg_main}.



\begin{algorithm}
\caption{Main Algorithm}\label{alg_main}
    \hspace*{\algorithmicindent} \textbf{Input:} Trajectory $\{\bv_n\}_{n = 1}^M$ observed on discrete time grid $\{\bt_n\}_{n = 1}^M.$\\
    \hspace*{\algorithmicindent} \textbf{Output:} Estimated memory kernel $\theta$ and $\theta_L.$
\begin{algorithmic}[1]
\STATE{Estimate $\bh$ using regularized Prony method (Algorithm \ref{alg_reg_Prony}).}
\IF{$F = 0$}
\STATE{Estimate $g$ using $\widetilde{g}$ derived from the derivative of $\widetilde{h}$ \eqref{eq_widetilde_dh}.}
\STATE{Achieve $\theta_L$ from \eqref{eq_theta_prony_frac_decomp} using fractional decomposition and explicit inverse Laplace transform.}
\ELSE
\STATE{Estimate $\varphi$ using Algorithm \ref{alg_reg_Prony} and derive $\widetilde{g}$ from \eqref{eq_prony_g}.}
\STATE{Achieve $\theta_L$ from \eqref{eq_theta_prony_frac_decomp_F_not_0}.}
\ENDIF
\STATE{Construct $\bA$, $\bb$ using \eqref{eq_lsq_A_b} given a basis $\{\psi_k\}_{k = 1}^K$.}
\STATE{Solve $\bc$ using \eqref{eq_lsq_RKHS_c} and achieve $\theta=\sum_{k = 1}^K\bc_k \psi_k$.}
\end{algorithmic}
\end{algorithm}

\section{Identifiability}\label{sec_id}
In this section, we prove the coercivity of the loss defined in \eqref{eq_loss_Sobolev_norm}, and control the estimation error in $\gamma$ by the error in the estimation of $h$ and $g$. The technique we used here is the Plancherel Theorem of Laplace transform as in Lemma \ref{eq_Plancherel_Laplace_exponential_decay}.

\subsection{Sobolev loss functions}
Suppose $g$ and $h$ are the true correlation functions, and the true kernel $\gamma$ is a solution for the equation \eqref{eq_main_volterra} and \eqref{eq_Volterra_second_kind}, from which we substitute $g$ and $g'$ into the loss function,
\begin{equation}\label{eq_loss_mE}
	\mE(\theta) = \int_{0}^{+\infty} \rbracket{\alpha_1\abs{(\gamma - \theta)*h}^2 + \alpha_2\abs{\rbracket{(\gamma - \theta)*h}'}^2} \rho(t) dt.
\end{equation}
We will make two assumptions before we give an analysis of this loss function. Firstly, a decaying and smoothness assumption on the kernel $\gamma$ and the observation data $h$ and $g$.
\begin{assumption}\label{asmp_exponential_dacay}
	The true kernel $\gamma$, the candidate kernel $\theta$, the correlation functions $h$ and $g$, and their noisy estimation $h_\varepsilon$ and $g_\varepsilon$ are smooth. These functions and their derivatives has exponential decay with parameter $\sigma>0$.
\end{assumption}
Let $\widehat{f}$ be the Laplace transform of a function $f$. For $\omega > 0$, denote that 
	\begin{equation}\label{eq_Lap_f_bound}
		M^f_\omega = \sup_{z = \omega + i\tau, \tau \in \R}\rbracket{\alpha_1\abs{\widehat{f}(z)}^2 + \alpha_2\abs{z\widehat{f}(z)}^2}, \quad m^f_\omega = \inf_{z = \omega + i\tau, \tau \in \R}\rbracket{\alpha_1\abs{\widehat{f}(z)}^2 + \alpha_2\abs{z\widehat{f}(z)}^2 }	
	\end{equation}
	We make the next assumption on $\gamma$, $h$ and $h_\varepsilon$. 
\begin{assumption}\label{asmp_omega_bdd}
	 There exists $\omega > 0$ such that $m_\omega^\gamma$, $m_\omega^h$, $m_\omega^{h_\varepsilon} >0$ and $M_\omega^\gamma$, $M_\omega^h$, $M_\omega^{h_\varepsilon} <\infty$.
\end{assumption}
In the above notations, $\omega$, $\gamma$, $h$ and $h_\varepsilon$ indicates the correspondence as in \eqref{eq_Lap_f_bound}.

We first prove the coercivity condition of the loss function.
\begin{theorem}\label{thm_coercivity}
	Under Assumption \ref{asmp_exponential_dacay} and Assumption \ref{asmp_omega_bdd}, 
	we have 
	\begin{equation}\label{eq_coercivity}
		m_\omega^h \norm{\theta - \gamma}^2_{L^2(\rho)}\leq \mE(\theta)  \leq M_\omega^h \norm{\theta - \gamma}^2_{L^2(\rho)}
	\end{equation}
	where \reva{$\rho(t) = e^{-2\omega t}dt$} and supported on $\R^+$. Hence $\gamma$ is identifiable by $\mE$ in $L^2(\rho)$, i.e. $\mE(\theta) = 0$ implies $\theta - \gamma = 0$ in $L^2(\rho)$.
\end{theorem}
\begin{proof}
	Let $f(t) = \theta(t) - \gamma(t)$ and $\Lambda(t) = \int_0^t f(s)h(t-s)ds$. Then we have $f$, $\Lambda$ and $\Lambda'$
	have exponential decay with parameter $\sigma$. By the Lemma \ref{eq_Plancherel_Laplace_exponential_decay}, using the notation $z = \omega + i\tau$, we have 
	\begin{align*}
		\mE(\theta) &= \int_0^{+\infty} \rbracket{\alpha_1\abs{\Lambda(t)}^2 + \alpha_2\abs{\Lambda'(t)}^2} e^{-2\omega t} dt  
		= \frac{1}{2\pi}\int_{-\infty}^{+\infty} \alpha_1\abs{\widehat \Lambda(z)}^2 + \alpha_2\abs{\widehat{\Lambda'}(z)}^2 d \tau
	\end{align*}
	Since $\widehat{\Lambda}(z) = \widehat h(z)\widehat f(z)$ and $\widehat{\Lambda'}(z) = z\widehat{\Lambda}(z) - \Lambda(0) = z\widehat{h}(z)\widehat{f}(z)$, 
\begin{equation}
		\mE(\theta) = \frac{1}{2\pi}\int_{-\infty}^{+\infty} \rbracket{\alpha_1\abs{\widehat h(z)}^2 + \alpha_2\abs{z\widehat{h}(z)}^2}\abs{\widehat f(z)}^2 d \tau
\end{equation}
The results follow from the bounds of $\alpha_1\abs{\widehat h(z)}^2 + \alpha_2\abs{z\widehat{h}(z)}^2$ in Assumption \ref{asmp_omega_bdd} and using Lemma \ref{eq_Plancherel_Laplace_exponential_decay} again. The identifiability is a natural result from \eqref{eq_coercivity}.
\end{proof}

In practice, we only have the noisy estimations of $g$ and $h$, denoted as $g_\varepsilon$ and $h_\varepsilon$. We use $\mE_{g_\varepsilon, h_\varepsilon}$ to denote the loss function when observing $g_\varepsilon$ and $h_\varepsilon$, and similarly for $\mE_{g_\varepsilon, h}$. The minimizer of  $\mE_{g_\varepsilon, h_\varepsilon}$ and $\mE_{g_\varepsilon, h}$ are denoted by $\gamma_{g_\varepsilon, h_\varepsilon}$ and $\gamma_{g_\varepsilon, h}$ respectively. Such solutions exist because of Assumption \ref{asmp_exponential_dacay}.
\begin{remark}\label{rmk_coercivity_h_eps}
	With the bound for $h_\varepsilon$ as in Assumption \ref{asmp_omega_bdd}, we have a similar coercivity condition for $\mE_{g, h_\varepsilon}$, namely
	\begin{equation}\label{eq_coercivity_error}
		m_\omega^{h_\varepsilon} \norm{\theta - \gamma_{g, h_\varepsilon}}^2_{L^2(\rho)}\leq \mE_{g, h_\varepsilon}(\theta)  \leq M_\omega^{h_\varepsilon} \norm{\theta - \gamma_{g, h_\varepsilon}}^2_{L^2(\rho)}
	\end{equation}
\end{remark}

Given the coercivity condition, we can prove the convergence of the estimation error in $\gamma$ with respect to the estimation error of $g$ and $h$. This is equivalent to saying the estimation of $\gamma$ is well-posed.  

\begin{theorem}\label{thm_main_error_conv}
	With the notations defined above, 
	\begin{enumerate}
		\item When we have the true $h$, the $L^2(\rho)$ error in the estimator is bounded by the error in $g$.	
		\begin{equation}
			\norm{\gamma - \gamma_{g_\varepsilon, h}}_{L^2(\rho)}^2 \leq \frac{1}{m_\omega^h}\norm{g - g_\varepsilon}^2_{H^1_{\boldalpha}(\rho)}
		\end{equation}
		\item With estimated $h_\varepsilon$ and $g_\varepsilon$, we have 
	\begin{equation}
		\norm{\gamma - \gamma_{g_\varepsilon, h_\varepsilon}}_{L^2(\rho)}^2 \leq \frac{2}{m_\omega^{h_\varepsilon}}\left(M^\gamma_\omega\norm{h - h_\varepsilon}_{L^2(\rho)}^2 + \norm{g - g_\varepsilon}^2_{H^1_{\boldalpha}(\rho)}\right)
	\end{equation}
	\end{enumerate}
\end{theorem}
\begin{proof}
	1. In the first case, since $\gamma_{g_\varepsilon, h}$ satisfies
	\begin{equation}
		g_\varepsilon(t) = (\gamma_{g_\varepsilon, h}*h)(t)	
	\end{equation}
	we have
	\begin{align*}
		\mE(\gamma_{g_\varepsilon, h}) 
		&= \alpha_1\int_{0}^{+\infty} \abs{g(t) - (\gamma_{g_\varepsilon, h}*h)(t)}^2 \rho(t) dt + \alpha_2\int_{0}^{+\infty} \abs{g'(t) - (\gamma_{g_\varepsilon, h}*h)'(t)}^2 \rho(t) dt\\
		&= \alpha_1\int_{0}^{+\infty} \abs{g(t) - g_\varepsilon(t)}^2 \rho(t) dt + \alpha_2\int_{0}^{+\infty} \abs{g'(t) - g_\varepsilon'(t)}^2 \rho(t) dt = \norm{g - g_\varepsilon}^2_{H^1_{\boldalpha}(\rho)}
	\end{align*}
	And by the coercivity condition \eqref{eq_coercivity}, 
	\begin{equation}
		\norm{\gamma - \gamma_{g_\varepsilon, h}}^2_{L^2(\rho)} \leq \frac{1}{m_\omega^h}\mE(\gamma_{g_\varepsilon, h}) = \frac{1}{m_\omega^h}\norm{g - g_\varepsilon}_{H^1(\rho)}^2
	\end{equation}
	
	2. In the second case, since
	\begin{equation}
		\norm{\gamma - \gamma_{g_\varepsilon, h_\varepsilon}}_{L^2(\rho)}^2 \leq 2\norm{\gamma - \gamma_{g, h_\varepsilon}}_{L^2(\rho)}^2 + 2\norm{\gamma_{g, h_\varepsilon} - \gamma_{g_\varepsilon, h_\varepsilon}}_{L^2(\rho)}^2
	\end{equation} 
	For the second term, we recall the bounded condition for $h_\varepsilon$ in Assumption \ref{asmp_omega_bdd} and the coercivity condition \eqref{eq_coercivity_error} in Remark \ref{rmk_coercivity_h_eps}. A similar step as in the previous case implies the following estimate.
	\begin{equation}
		\norm{\gamma_{g, h_\varepsilon} - \gamma_{g_\varepsilon, h_\varepsilon}}_{L^2(\rho)}^2 \leq 
		\frac{1}{m^{h_\varepsilon}_\omega}
		\mE_{g, h_\varepsilon}(\gamma_{g_\varepsilon, h_\varepsilon}) 
        = 
        \frac{1}{m^{h_\varepsilon}_\omega}\norm{g - g_\varepsilon}_{H^1_{\boldalpha}(\rho)}^2
	\end{equation}
	For the first term, by the coercivity condition \eqref{eq_coercivity},
	\begin{equation}
		\norm{\gamma - \gamma_{g, h_\varepsilon}}_{L^2(\rho)}^2 \leq \frac{1}{m^{h_\varepsilon}_\omega}\mE_{g, h_\varepsilon}(\gamma)
	\end{equation}
    For the second term, notice that 	
	\begin{align*}
		\mE_{g, h_\varepsilon}(\gamma) 
		&= \alpha_1\int_{0}^{+\infty} \abs{g(t) - (\gamma*h_\varepsilon)(t)}^2 \rho(t) dt + \alpha_2\int_{0}^{+\infty} \abs{g'(t) - (\gamma*h_\varepsilon)'(t)}^2 \rho(t) dt\\
		&= \alpha_1\int_{0}^{+\infty} \abs{(h - h_\varepsilon)*\gamma}^2 \rho(t) dt + \alpha_2\int_{0}^{+\infty} \abs{((h - h_\varepsilon)*\gamma)'}^2 \rho(t) dt
	\end{align*}
	Because of the bounds for $\gamma$ in Assumption \ref{asmp_omega_bdd}, taking $h - h_\varepsilon$ as $f$ and $\gamma$ as $h$ in the proof of Theorem \ref{thm_coercivity}, we can have a similar upper bound as in \eqref{eq_coercivity}, namely
	\begin{equation}
		\mE_{g, h_\varepsilon}(\gamma) \leq M^\gamma_\omega
		\norm{h - h_\varepsilon}^2_{L^2(\rho)}
	\end{equation}
	And then using the coercivity condition \eqref{eq_coercivity_error} in Remark \ref{rmk_coercivity_h_eps},
	\begin{equation}
		\norm{\gamma - \gamma_{g, h_\varepsilon}}_{L^2(\rho)}^2 \leq \frac{M^\gamma_\omega}{m_\omega^{h_\varepsilon}}\norm{h - h_\varepsilon}^2_{L^2(\rho)}
	\end{equation}
	Finally,
	\begin{equation}
		\norm{\gamma - \gamma_{g_\varepsilon, h_\varepsilon}}_{L^2(\rho)}^2 \leq \frac{2}{m_\omega^{h_\varepsilon}}\left(M^\gamma_\omega\norm{h - h_\varepsilon}_{L^2(\rho)}^2 + \norm{g - g_\varepsilon}^2_{H^1_{\boldalpha}(\rho)}\right). 
	\end{equation}
\end{proof}

Notice that the estimation error $\norm{\gamma - \gamma_{g_\varepsilon, h_\varepsilon}}_{L^2(\rho)}$ scales linearly with the estimation error in $\norm{g - g_\varepsilon}_{H^1(\rho)}^2$ and $\norm{h - h_\varepsilon}_{L^2(\rho)}^2$, but the constant is inversely proportional to $m_\omega^{h_\varepsilon}$. To achieve a better convergence, we have to make sure $m_\omega^{h_\varepsilon}$ is away from 0. Such a condition can be derived from Prony estimations, i.e. using $\widetilde{h}$ as the estimated $h_\varepsilon$. To emphasize this advantage, we compare the current loss function with two other choices of loss functions.

\subsection{Ill-posed loss functions}
It is natural to consider the loss function directly based on the Volterra equation \eqref{eq_main_volterra}, 
\begin{equation}\label{eq_loss_space_old_old_old}
		\mE_1(\theta) = \int_{0}^{+\infty} \abs{g(t) - \int_0^t \theta(t-s)h(s)ds}^2 \rho(t) dt = \int_{0}^{+\infty} \abs{(\gamma - \theta)*h}^2 \rho(t) dt
\end{equation}
Since the second kind Volterra equation is well-posed, we also construct the following loss function based on \eqref{eq_Volterra_second_kind},
\begin{equation}\label{eq_loss_derivative}
		\mE_2(\theta) = \int_{0}^{+\infty} \abs{g'(t) - \left(\int_0^t \theta(t-s)h'(s)ds  - h(0)\theta(t)\right)}^2 \rho(t) dt = \int_{0}^{+\infty} \abs{((\gamma - \theta)*h)'}^2 \rho(t) dt
\end{equation}
The corresponding estimators using loss functions $\mE_1$ and $\mE_2$ are denoted as $\theta_1$ and $\theta_2$. As we have demonstrated in the proof of the Theorem \ref{thm_main_error_conv}, the crucial estimate for the loss function to be well-posed is when $h$ is replaced by the estimated form $h_\varepsilon$, the following coercivity condition holds,
\begin{equation}
	\mE_1(\theta) \geq m_1 \norm{\theta - \gamma}^2_{L^2(\rho)}, \ \mE_2(\theta) \geq m_2 \norm{\theta - \gamma}^2_{L^2(\rho)}
\end{equation}
for some constants $m_1, m_2 >0$. However, we show that this cannot be achieved for $\mE_1$, and for $\mE_2$, the constant $m_2$ scales badly with $\omega$. 
\begin{proposition}\label{prop_mE1_mE2}
	Suppose $\mE_1$ and $\mE_2$ are defined as above, and $h$ is approximated by a Prony series of the form $h_\varepsilon(t) = \sum_{k = 1}^p \bw_k e^{\boldlambda_k t}$ with $\boldlambda_k < -\sigma$, $k = 1, \dots, p$. Then such $m_1>0$ does not exist and $m_2 = O(\omega^2)$ as $\omega \to 0$.  
\end{proposition}
\begin{proof}
	By the Plancherel Theorem \ref{lem_Plancherel} and similar steps in the proof of Theorem \ref{thm_coercivity}, we have 
\begin{align}
		\mE_1(\theta) = \frac{1}{2\pi}\int_{-\infty}^{+\infty} \abs{\widehat{h_\varepsilon}(z)}^2\abs{\widehat f(z)}^2 d \tau, \ \ \ \ 
		\mE_2(\theta) = \frac{1}{2\pi}\int_{-\infty}^{+\infty} \abs{z\widehat{h_\varepsilon}(z)}^2\abs{\widehat f(z)}^2 d \tau
\end{align}
where $z = \omega + i\tau$ and $f = \theta - \gamma$. Based on the given form of $h_\varepsilon$, we have $\widehat{h_\varepsilon}(z) = \sum_{k = 1}^p \frac{\bw_k}{z - \boldlambda_k}$. It is clear that 
\begin{equation}\label{eq_mE1_limit_0}
	\lim_{\tau \to \infty} \abs{\widehat{h_\varepsilon}(\omega + i\tau)} = 0.
\end{equation}
hence, a positive lower bound $m_1$ does not exist. And in the second case, $z\widehat{h_\varepsilon}(z) = \sum_{k = 1}^p \frac{z\bw_k}{z - \boldlambda_k}$, which is not $0$ when $\tau \to \infty$. It is possible that there exists $m_2>0$. However, 
\begin{equation}\label{eq_mE2_limit_0}
	\lim_{z \to 0} \abs{z\widehat{h_\varepsilon}(z)} = 0,
\end{equation}
hence $m_2 = \inf_{z = \omega + i\tau, \tau \in \R}\abs{z\widehat{h_\varepsilon}(z)}^2 \sim O(\omega^2)$ when $\omega \to 0$. 
\end{proof}
The fact that $m_1$ does not exist prevents us from using the loss function $\mE_1$. And even though $m_2$ exists when $\omega >0$, it scales badly when $\omega$ is small. Since the density for $\rho$ is $e^{-2\omega t}$, our confidence about the estimated kernel $\theta$ is constrained near the origin when $\omega$ is large. Therefore $\mE_2$ is not able to provide a reasonable performance guarantee for the value of the estimated kernel away from the origin.

From the observations of $\mE_1$ and $\mE_2$, the proposed loss function $\mE$ is constructed as their linear combination,
\begin{equation}
	\mE(\theta) = \alpha_1\mE_1(\theta) + \alpha_2\mE_2(\theta)
\end{equation}
And for $\mE$, the minimum in Assumption \ref{asmp_exponential_dacay} is achievable, since the function $\abs{z\widehat{h_\varepsilon}(z)}^2 + \abs{\widehat{h_\varepsilon}(z)}^2$ does not have zeros at the origin, nor infinity.  We will show this fact using numerical examples in the next section.

\begin{remark}\label{rmk_alpha}
	The above construction of $\mE$ gives us some insight into selecting the optimal parameter $\boldalpha$. Because of equation \eqref{eq_mE1_limit_0} and \eqref{eq_mE2_limit_0}, and notice that 
	\[\alpha_1' := \lim_{z \to \infty}z\widehat{h_\varepsilon}(z) = \sum_{k = 1}^p \bw_k = h_\varepsilon(0), \ \ \alpha_2' := \lim_{z \to 0}\widehat{h_\varepsilon}(z) = \sum_{k = 1}^p \frac{\bw_k}{-\boldlambda_k} = \int_0^\infty h_\varepsilon(t)dt\]
	Therefore in practice, we let $\boldalpha = (\alpha_1, \alpha_2) = \frac{1}{\alpha_1'^2 + \alpha_2'^2}(\alpha_1', \alpha_2')$. Such an arrangement makes 
	\[\lim_{z \to \infty}\alpha_1 \abs{\widehat{h_\varepsilon}(z)}^2 + \alpha_2 \abs{z\widehat{h_\varepsilon}(z)}^2= \lim_{z \to 0}\alpha_1 \abs{\widehat{h_\varepsilon}(z)}^2 + \alpha_2 \abs{z\widehat{h_\varepsilon}(z)}^2 = \frac{\alpha_1'\alpha_2'}{\alpha_1'^2 + \alpha_2'^2}, \]
	so that optimally balances the value of the above function at 0 and $\infty$. However, there is no guarantee that the minimum of $\alpha_1 \abs{\widehat{h_\varepsilon}(z)}^2 + \alpha_2 \abs{z\widehat{h_\varepsilon}(z)}^2$ is achieved at 0 and $\infty$, but a strictly positive minimum is possible to exist. 
\end{remark}

\section{Numerical Results}\label{sec_numerics}
We justify our analysis and examine the performance of our algorithm in the following examples. First, we give a typical learning result with the force $F$ equal to 0, demonstrating the algorithm and comparing the prediction error on the correlation functions. We also compare the performance with the proposed loss function $\mE$ and the ill-posed ones $\mE_1$ and $\mE_2$. Next, we give the convergence of estimation error, both theoretical and empirical, with respect to the observation noise on the trajectory and the changing decaying speed of the measure $\rho$. Finally, we present the case of $F$ not being 0 and an example of chaotic dynamics with an extra drift term. In the last two cases, the solution $v$ is not ergodic. Hence, the ensemble trajectory data becomes necessary. 
\subsection{Typical learning result when $F = 0$}
We take the true memory kernel to be $\gamma(t) = \sum_{k = 1}^{5}\bu_k e^{\boldeta_k t}$, with 

$\bu = [0.3488, 0.3488, 0.3615, 0.5300, 0.3045]$ and $\boldeta =  [-0.1631 - 0.3211i, -0.1631 + 0.3211i, -0.8262 + 0.0000i, -0.9178 + 0.0000i, -0.3352 + 0.0000i]$ sampled from random and ensured exponential decay. The correlated noise is generated using spectral representation \cite{shinozukaSimulationStochasticProcesses1991}, with accelerated convergence \cite{huSimulationStochasticProcesses1997}. The spectral grid we used is $N_w = 10000$, $\Delta w = \pi/100$, the length of the time grid to be $L = 2^{16}$ with the inverse temperature $\beta = 1$. Hence the resulting correlated noise $R(t)$ is evaluated on a grid with $\Delta t = \frac{\pi}{N_w\Delta w} = 0.01$ and an artificial period $T_0 = \frac{2\pi}{\Delta w} = 200$, denoted as $\bR_l = R(\bt_l)$, where $\bt_l = l {\Delta t}, l = 0, 1, \dots, L$. The trajectory is generated using the Euler scheme, with $\bv_0 \sim \mN(0, 1)$ and 
\begin{equation}
	\bv_{l+1} = \bv_{l} + \Delta t\left[F(\bv_l) - \sum_{j = 0}^l\gamma({\bt}_{l-j})\bv_j + \bR_l\right]
\end{equation}
We call $\bv$ the true trajectory, and the noisy observation of the trajectory is given by
\begin{equation}
	\widetilde\bv_l = \bv_{lr} + \boldsymbol{\xi}_l
\end{equation}
where $r = 70$ is the observation ratio, and $\boldxi_n$ are i.i.d Gaussian random noise with mean 0 and standard deviation $\sigma_{obs} = 0.1$. We also assume a shorter total observation length of $\widetilde L = L/(2r)$.

\begin{figure}[tbhp]
  \centering
  \subfloat{\label{fig:F_0_kernel}\includegraphics[width=0.33\textwidth]{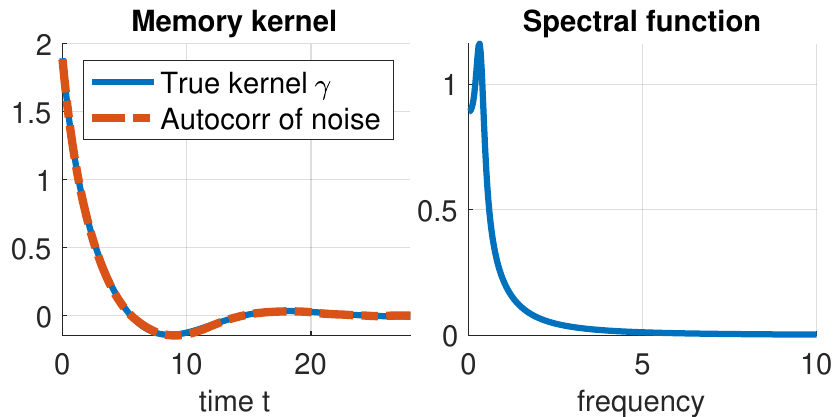}}
  \subfloat{\label{fig:F_0_noise}\includegraphics[width=0.66\textwidth]{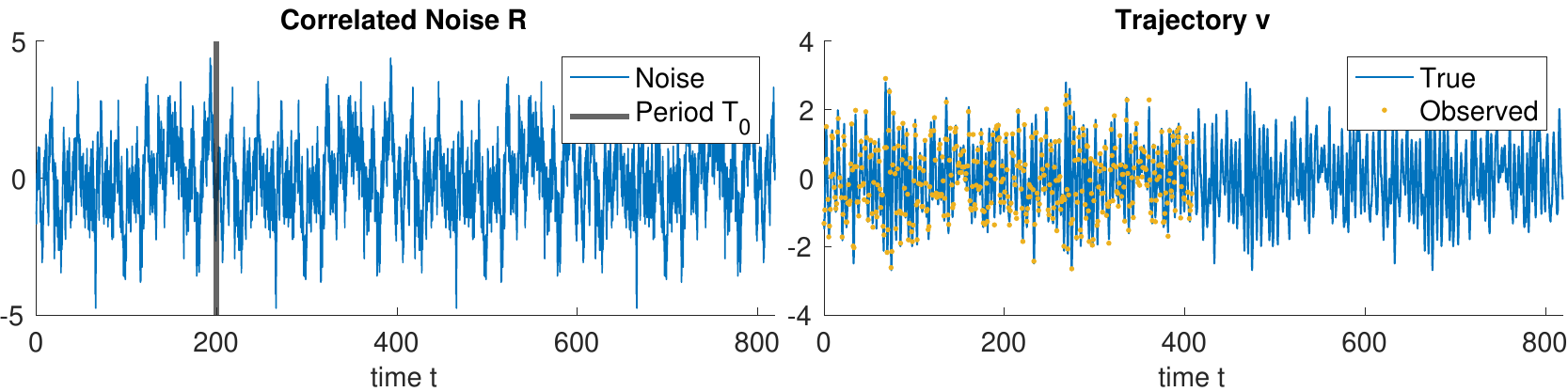}}\\
  \subfloat{\label{fig:F_0_traj}\includegraphics[width=1\textwidth]{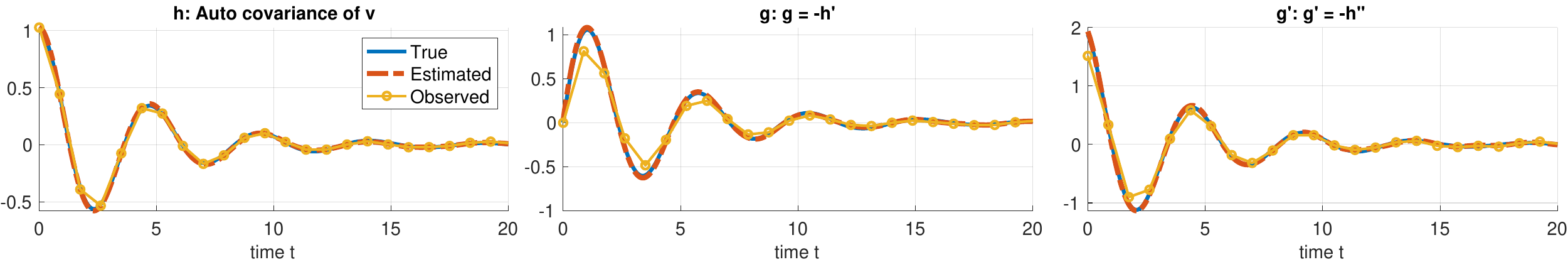}}
  \caption{Memory kernel, trajectory, and correlation functions. \textbf{First row:} Panel 1: Memory kernel $\gamma$, and the autocorrelation of simulated noise to verify the fluctuation-dissipation condition \eqref{eq_fluctuation_dissipation}. Panel 2: Spectral function of $\gamma$. Panel 3: Trajectory of correlated noise $R$. Panel 4: Trajectory $v$ and a discrete and noisy observation. \textbf{Second row:} Estimation result of correlation functions $h$, $g$ and $g'$. In the case of $F\equiv 0$, we have $g = h'$. Taking derivatives on the estimated $h$ (named "Estimated $g$" and "Estimated $g'$") has better performance of $g$ and $g'$ than estimating from the finite difference of the observed trajectory $\widetilde{\bv}$ (named "Observed $g$" and "Observed $g'$"). \\}
  \label{fig:F_0}
\end{figure}

%
%

The true kernel is shown in the first panel of Figure \ref{fig:F_0}, together with the estimated autocorrelation of the noise sequence, as a way to check the quality of the simulated $\bR$. The spectral function, which is the Fourier transform of the memory kernel $\gamma$, is presented in the second panel for later comparison of prediction error. The trajectory and the noise sequence are shown on the right of the first row in Figure \ref{fig:F_0}. 

The estimation result of correlation functions is presented in the second row of Figure \ref{fig:F_0}.  In the case of $F = 0$, we have $g = -h'$. Note that the true $h$ can be derived from Remark \ref{rmk_gamma_h} with 6 Prony modes. From the observed trajectory, we estimate $\{\bh_n\}_{n = 1}^N$ with $N = 24$ and $p = 10$, which is different from the true number of modes 6 to show the robustness of the regularized Prony method. 

Note that the observed $\bh_n$ is close to the true $h$, which is estimated from the full trajectory. However, estimating the derivatives of $h$ using the finite difference of $\bh_n$ is problematic, especially for the region where the function varies rapidly, as shown in the middle and the right panel of the second row in Figure \ref{fig:F_0}. We point out that the Prony estimation $\widetilde{h}$ is close to both $\bh_n$ and true $h$, and taking the explicit derivatives of the Prony estimation $\widetilde{h}$ provides us a better estimation of its derivatives. 

For later comparison of the theoretical upper bound, we take $\omega = 0.05$ and then \reva{$\rho(t) = e^{-0.1t}$}. The scale parameter $\boldalpha = (0.9030, 0.0970)$, which derived from Remark \ref{rmk_alpha}. In this setting, the squared estimation error of $h$ and $g$ are $\|h - \widetilde h\|_{L^2(\rho)}^2 = 1.12\times 10^{-3}$ and $\|g - \widetilde g\|_{H^1_{\boldalpha}(\rho)}^2 =  2.37\times 10^{-3}$.

We estimate the memory kernel using the proposed estimator $\theta$, which minimizes the Sobolev loss function $\mE$, and compare the result with estimators $\theta_1$ which minimizes $\mE_1$, $\theta_2$ which minimizes $\mE_2$, and $\theta_L$ which is derived explicitly by inverse Laplace transform. The basis we used here is cubic spline functions over the interval $[0, 30]$ with 30 knots. The estimation result of the memory kernel is shown in Figure \ref{fig:F_0_kernel_est}. The estimators $\theta_1$ and $\theta_2$ are worse than our proposed estimator $\theta$ which using $\mE$, and the performance of $\theta_L$ is comparable to $\theta$. 

\begin{figure}[h]
		\centering
		\includegraphics[width=0.75\textwidth]{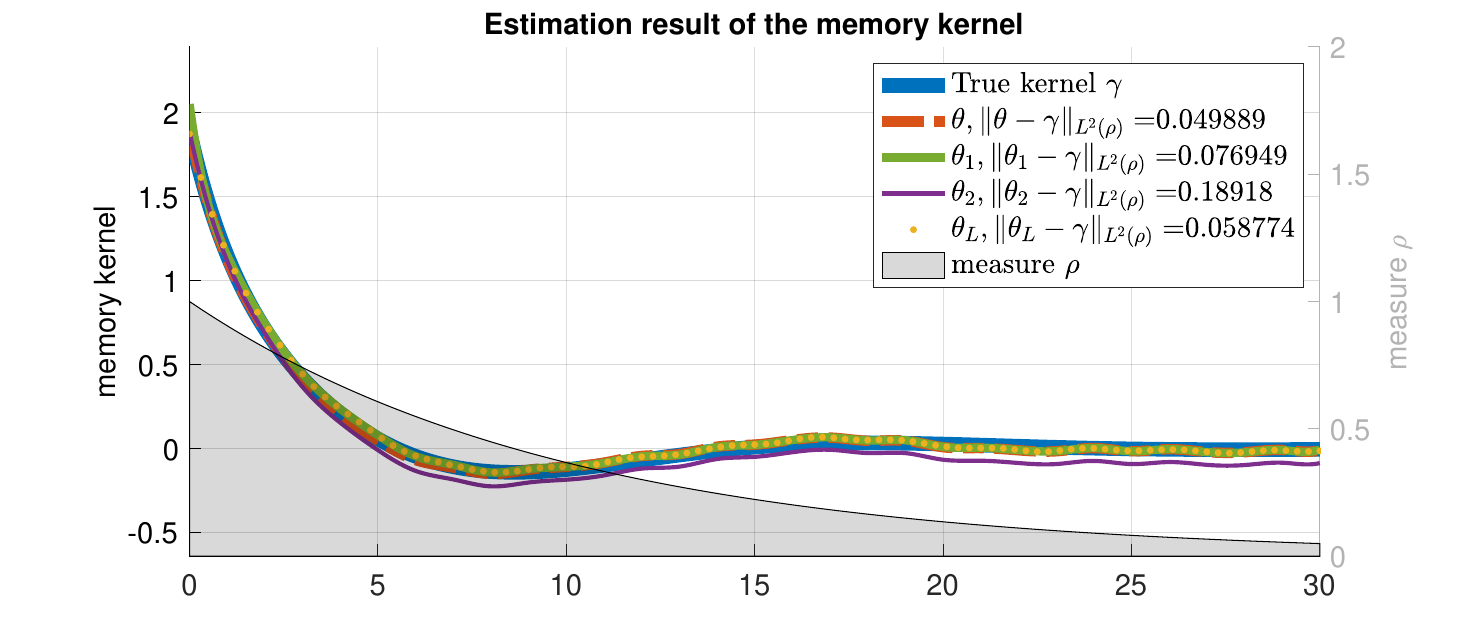}
		\caption{Estimation result of the memory kernel. The proposed estimator $\theta$, which minimizes the Sobolev loss function, has the best performance, compared with the minimizer $\theta_1$, $\theta_2$ of ill-posed loss functions $\mE_1$ and $\mE_2$, and the explicit inverse Laplace transform estimator $\theta_L$.}
		\label{fig:F_0_kernel_est}
\end{figure}

Now we give a numerical comparison of the upper and lower bounds $M_\omega^\gamma$, $m_\omega^{h_\varepsilon}$, $m_1$, $m_2$, as they are crucial in estimating the coercivity constant in Theorem \ref{thm_coercivity} and Proposition \ref{prop_mE1_mE2}. The result is presented in the first panel of Figure \ref{fig:F_0_coer_pred_err}. We fix the $\omega = 0.05$ as above and examine the values of the above functions on the grid $z = \omega + i\tau$, where $\tau$ is exponentially evenly distributed from $10^{-5}$ to $10^{5}$, and plot the functions $\alpha_1|\widehat{\gamma(z)}|^2+ \alpha_2 |z\widehat{\gamma(z)}|^2$,
$|\widehat{h_\varepsilon}(z)|^2$, $|z\widehat{h_\varepsilon}(z)|^2$ and $\alpha_1|\widehat{h_\varepsilon}(z)|^2+ \alpha_2 |z\widehat{h_\varepsilon}(z)|^2$ to locate the upper and lower bounds. Note that  $|\widehat{h_\varepsilon}(z)|^2$ goes to 0 as $\tau$ increases, so its minimus $m_1 = 0$, showing the ill-posedness of $\mE_1$. Recall that $m_2=O(\omega^2)$ when $\omega$ goes to 0 and $|z\widehat{h_\varepsilon}(z)|^2$ achieves minimum $m_2$ for small $\tau$, hence $\mE_2$ is badly conditioned for small $\omega$.  By our construction of the loss function $\mE$ and the selection of the coefficient $\boldalpha$ as in Remark \ref{rmk_alpha}, the minimum of $\alpha_1|\widehat{h_\varepsilon}(z)|^2+ \alpha_2 |z\widehat{h_\varepsilon}(z)|^2$ is not achieved at 0 nor $\infty$, whereas its value at these two extreme points are balanced. The actual minimum $m_\omega^{h_\varepsilon} = 8.72\times 10^{-2}$ is located at the red star, and the theoretical lower bound of the estimation error $\|\theta - \gamma\|_{L^2(\rho)}^2$ in Theorem \ref{thm_coercivity} is 0.316, which controls the estimation error of $\theta$ in practice as shown in the Figure \ref{fig:F_0_coer_pred_err}.

Finally, we will provide a direct comparison of the prediction error. We use the learned kernel $\theta$ and $\theta_L$ to generate a new trajectory and compare the autocorrelation function $h$. The results are presented in the second and third panels of Figure \ref{fig:F_0_coer_pred_err}. Both $\theta$ and $\theta_L$ reconstruct the autocorrelation accurately, but the spectral function of $\theta_L$ contains false phase changes, which appear because of the ill-posedness of the inverse Laplace transform.

\begin{figure}[ht]
\includegraphics[width=\linewidth]{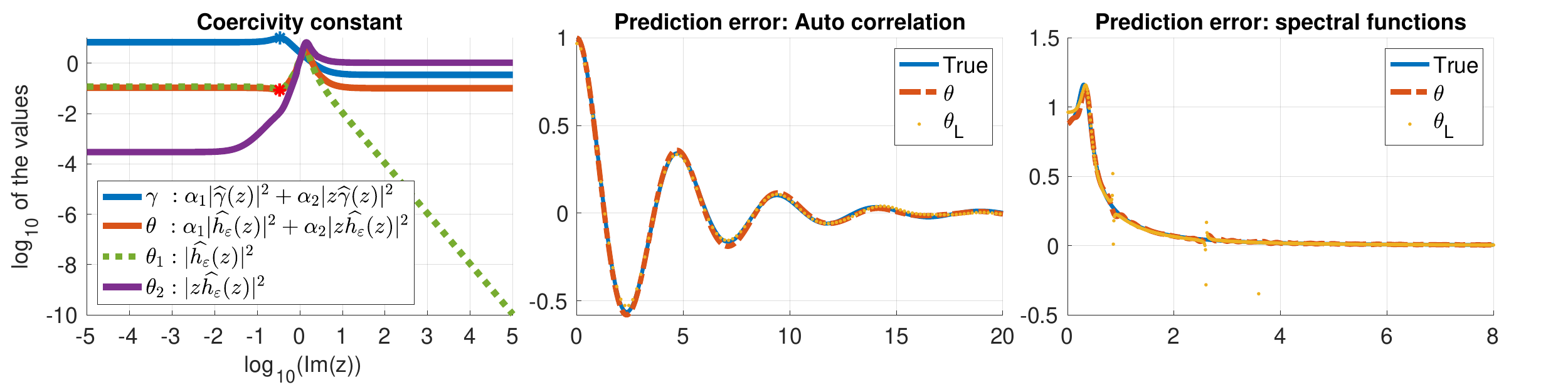}
\caption{Coercivity constants and prediction error. \textbf{Panel 1:} Locate the upper and lower bounds for coercivity constants. The lower bound $m_1$, corresponding to the estimator $\theta_1$, equals to 0. The lower bound $m_2$, corresponding to the estimator $\theta_2$, is achieved when $\text{Im}(z) = 0$, and has order $O(\omega^2)$ as $\omega \to 0$. The lower bound $m_\omega^{h_\varepsilon	}$, corresponding to $\theta$,  is achieved at the red star, whereas the value at extreme points 0 and $\infty$ are bounded by the selection of $\boldalpha$ in Remark \ref{rmk_alpha}. The scale of $\alpha_1|\widehat{\gamma(z)}|^2+ \alpha_2 |z\widehat{\gamma(z)}|^2$ varies mildly, and a blue star marks the upper bound. \textbf{Panel 2:} Prediction error of autocorrelation of the trajectory. \textbf{Panel 3:} Prediction error of the spectral functions of $\theta$ and $\theta_L$. }
\label{fig:F_0_coer_pred_err}
\end{figure}

\subsection{Performance with changing parameters}\label{sec:changing_parameters}
We examine the performance of the estimators with the change of the scaling parameter $\omega$ and the standard deviation of observation noise $\sigma_{obs}$. 

\paragraph{Performance with changing $\omega$}
Under the same numerical settings as above, we take $\omega$ to be exponentially evenly spaced on $[10^{-3}, 10^{2}]$, and compare the performance of estimators $\theta, \theta_L, \theta_1$ and $\theta_2$. The result is presented in Figure \ref{fig:change_omega}. In the first panel, we present the change of the $L^2(\rho)$ errors of the four estimators with $\omega$. The error of our proposed estimator $\theta$ is bounded by the theoretical upper bound in Theorem \ref{thm_main_error_conv}, which is presented as the solid blue line. The upper bound gets larger when $\omega$ gets smaller, for which the measure $\rho$ becomes more uniform, losing the concentration around the origin, hence making the estimation error of $h$ and $g$, and the upper bound $M_\omega^\gamma$ in Theorem \ref{thm_main_error_conv}, larger as shown in the third panel. However, we point out that the constant $m_\omega^{h_\varepsilon}$ does not change much for different scales of $\omega$, whereas the coercivity lower bound $m_2$ for $\mE_2$ as in Proposition \ref{prop_mE1_mE2}, grows as $O(\omega^2)$ when $\omega \to 0$. We could derive a similar upper bound for estimator $\theta_2$, but such an upper bound is impractical because it blows up as $\omega \to 0$.

The error of the estimator $\theta_L$ using explicit inverse Laplace transform of the Prony series exceeds the upper bound when $\omega$ is large, and so does the error of $\theta_1$, which uses the ill-posed loss function $\mE_1$. The error of $\theta_2$, derived from the loss function $\mE_2$, is also controlled by the theoretical upper bound but is outperformed by $\theta$ for the case that $\omega$ is small. The overall decreasing trends of the error as $\omega$ becomes larger is due to the decreasing of the $L^2(\rho)$ norm. We present the relative $L^2(\rho)$ error in the second panel to compensate for such a change. 

In the second panel of Figure \ref{fig:change_omega}, the decreasing tendency of the $\norm{\gamma}_{L^2(\rho)}$ is presented in the solid blue line. After normalization, the relative $L^2(\rho)$ error becomes more uniform across different scales of $\omega$. Note that \reva{$\rho(t) = e^{-2\omega t}$} represents our weight of goal in estimating the kernel $\gamma$. The minimum of the relative error of $\theta$ for $\omega$ around $10^{0.6}$, which indicates an optimal selection of $\omega$. However, this observation is only available oracularly because of the usage of the true kernel $\gamma$. We mention such an optimal value and leave it as future work.

\begin{figure}[h]
		\centering
		\includegraphics[width=1\textwidth]{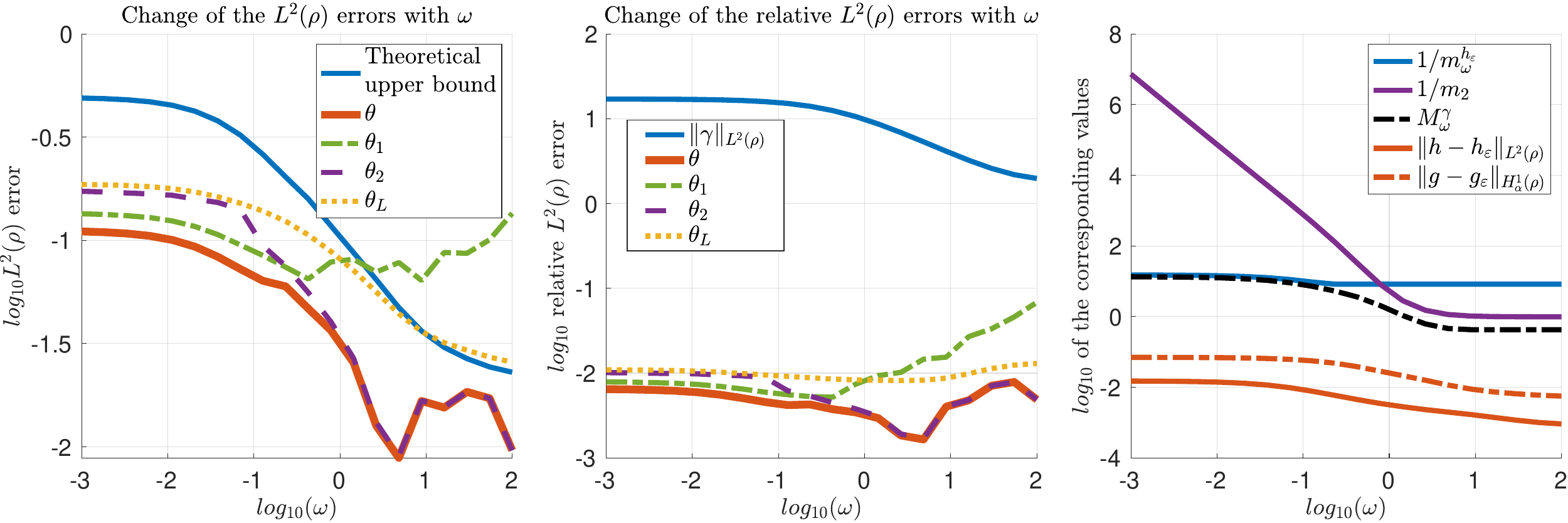}
		\caption{Performance with changing $\omega$. \textbf{Panel 1:} The $L^2(\rho)$ estimation error of $\theta, \theta_L, \theta_1$ and $\theta_2$ with changing $\omega$. The performance of $\theta$ is the best among the four and is controlled by the theoretical upper bound. \textbf{Panel 2:} The relative $L^2(\rho)$ estimation error. There exists an optimal selection of $\omega$ by assuming the knowledge of true kernel $\gamma$. \textbf{Panel 3:} Quantities relevant to the theoretical upper bound. The lower bound $m_\omega^{h_\varepsilon}$ changes mildly with $\omega$, but $1/m_2$ blows up for small $\omega$. } 
		\label{fig:change_omega}
\end{figure}

\paragraph{Performance  with changing $\sigma_{obs}$}
With the same numerical settings as above and let $\omega = 0.25$, we take $\sigma^2_{obs}$ to be exponentially evenly spaces on $[10^{-3}, 10^{1}]$, and compare the performance of estimators $\theta, \theta_L, \theta_1$ and $\theta_2$. The result is presented in Figure \ref{fig:change_obs_std}. The second panel shows that the error in the estimated $g$ and $h$ decreases as $\sigma_{obs}^2 \to 0$. Also, the coercivity constant $m_\omega^{h_\varepsilon}$ decreases when $\sigma_{obs}^2$ gets smaller but will reach a constant. The decreasing speed of squared estimation error in $h$ and $g$ is faster than the decreasing speed of $m_\omega^{h_\varepsilon}$. Therefore, the theoretical upper bound in the first panel decreases as $\sigma_{obs}^2 \to 0$ and reaches a constant value for small $\sigma_{obs}^2$. 

The estimation error of our estimators $\theta$ also decreases with $\sigma_{obs}^2$, bounded by the theoretical upper bound, implying the wellposedness of $\mE$. Notice that the error of $\theta_1$ does not decrease when $\sigma_{obs}^2$ is around $10^{-1.5}$, since the loss function $\mE_1$ is ill-posed. The decreasing phenomenon of the error of $\theta_2$ and $\theta_L$ results from regularization, but their performance lacks justification.

\begin{figure}[h]
		\centering
		\includegraphics[width=0.67\textwidth]{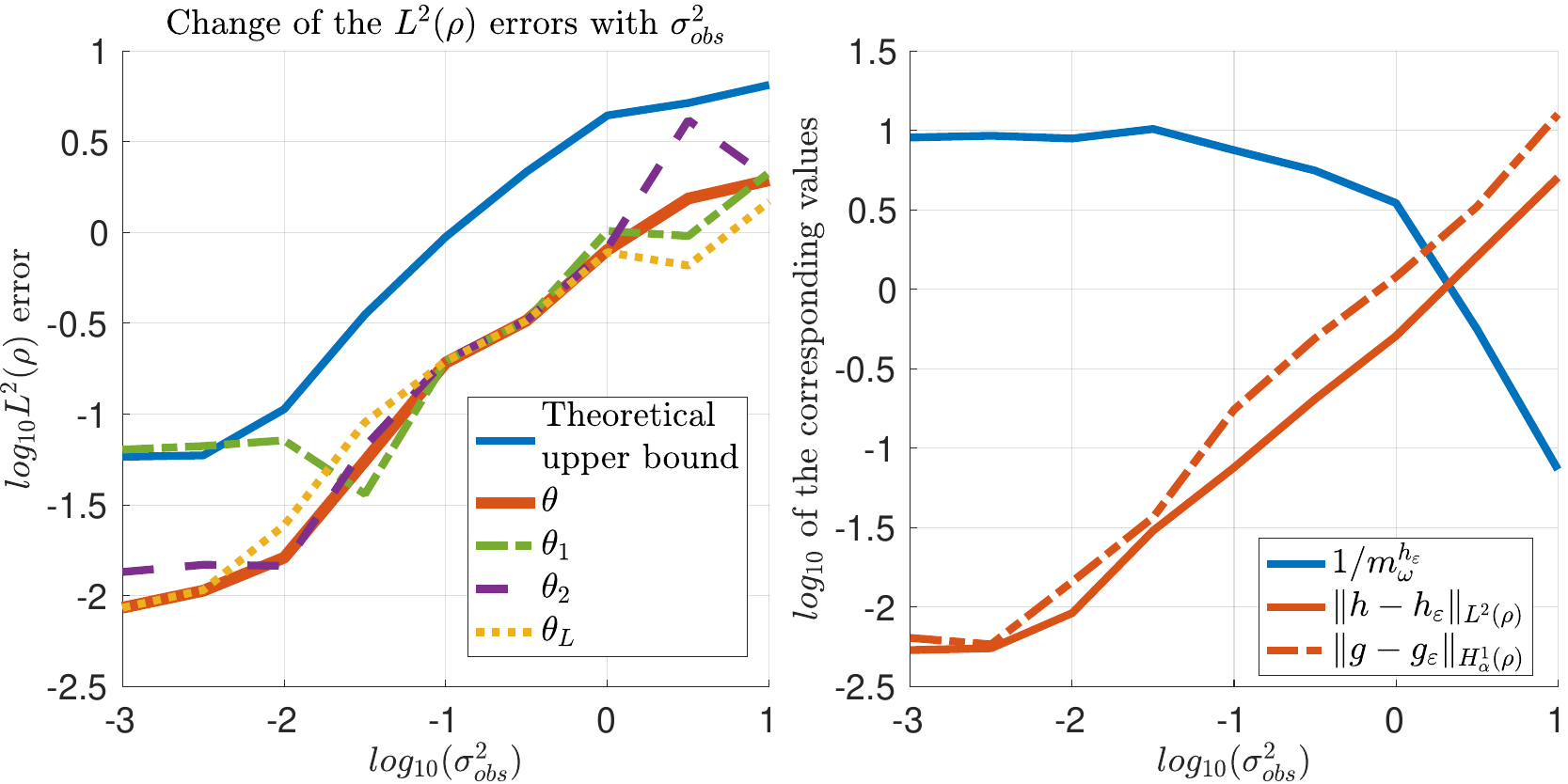}
		\caption{Performance with changing observation noise $\sigma_{obs}$. \textbf{Panel 1:} Estimation errors of $\theta, \theta_L, \theta_1$ and $\theta_2$ with changing $\sigma_{obs}^2$. Both the theoretical upper bound and the $L^2(\rho)$ error of $\theta$ decrease as $\sigma_{obs}$ decreases. \textbf{Panel 2:} Quantities relevant to the theoretical upper bound. The error in the estimated correlation decreases as $\sigma_{obs}$ decreases.}
		\label{fig:change_obs_std}
\end{figure}

\subsection{Examples with force term and drift}

\subsubsection{Double well potential}
\reva{We consider a system with a force derived from a double-well potential and a memory kernel exhibiting power-law decay,
\begin{equation}\label{eq_double_well_power_law}
	F(v) = -v(v^2 - 4), \quad \gamma(t) = \frac{1-3t^2}{(1+t^2)^3}.
\end{equation}
Since the true correlation functions $h$ and $g$ are not analytically available, we cannot provide a theoretical guarantee based on their estimation errors. Instead, we compare results across different data sizes. The external force and power-law memory kernel hinder mixing, leading to poor correlation estimates from temporal integrals. A natural remedy is to use ensemble independent trajectories for more accurate estimation.} Suppose our observed data is given by $\{\bv_{l}^{m}\}_{l = 1, m = 1}^{L, M}$ where $M$ represents independent trajectory samples. We estimate the correlation functions on discrete time grids by
\begin{equation}\label{eq_cov_func_ensemble}
	\bh_{n} = \frac{1}{(L-n)M}\sum_{m = 1}^M\sum_{l = 1}^{L-n} \bv_l^m \bv_{l+n}^m, \ \ \ \boldphi_{n} = \frac{1}{(L-n)M}\sum_{m = 1}^M\sum_{l = 1}^{L-n} \bv_l^m F(\bv_{l+n}^m), 
\end{equation}
then apply the regularized Prony method Algorithm \ref{alg_main} for interpolations and infer the memory kernel using Algorithm \ref{alg_main}. We compare the performance of one short trajectory $L = 2^{12}$, one long trajectory $L = 2^{16}$ and ensemble short trajectories $L = 2^{12}, M = 2000$. The results are illustrated in Figure \ref{fig:F_deep_double_well_G_0}. \reva{The experiment uses $\rho(t) = e^{-0.1t}$. The true trajectory is simulated with time step $\Delta t = 0.0125$, and observations are taken every $10\Delta t = 0.125$ with noise level $\sigma_{obs} = 0.01$. We estimate 30 discrete values of the autocorrelation functions and apply the Prony method with 10 modes for interpolation. For regression, we use cubic splines on $[0, 30]$ with 50 knots. The parameter $\balpha$ from Remark~\ref{rmk_alpha} is set to $(0.9725, 0.0275)$, $(0.9805, 0.0195)$, and $(0.9999, 0.0001)$ in the respective tests.

\begin{remark}\label{rmk_powerlaw}
 	The power-law kernel defined in \eqref{eq_double_well_power_law} does not have exponential decay. Therefore, our previous analysis does not hold. However, the same algorithm can be applied empirically to estimate the memory kernel. See the Appendix \ref{apdx} for a thorough study of learning the power-law memory kernel. 
\end{remark}

}

\begin{figure}[h]
		\centering
		\includegraphics[width=0.8\textwidth]{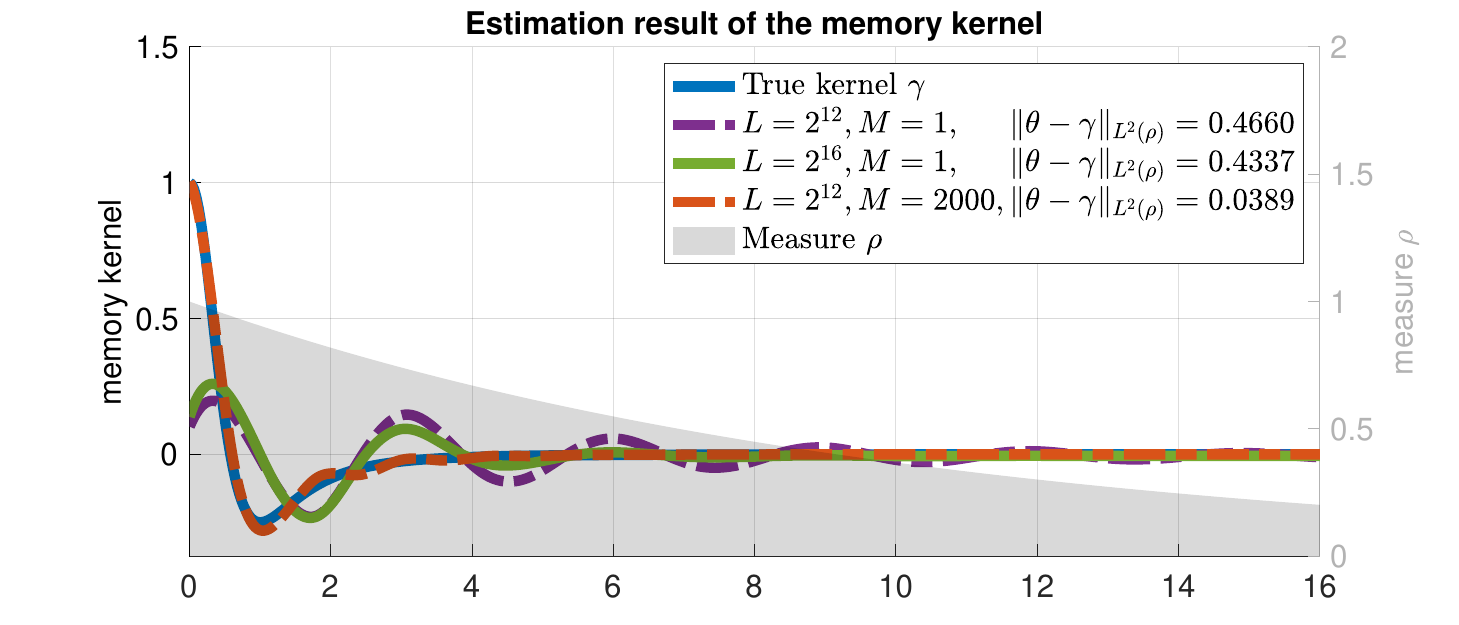}
		\caption{Estimation results for a system with a double-well potential and power-law memory kernel. Non-stationary dynamics lead to poor autocorrelation estimates from temporal averages, so increasing $L$ does not improve the estimation of $\gamma$. In contrast, a larger number of independent trajectories $M$ yields more accurate results.}
		\label{fig:F_deep_double_well_G_0}
\end{figure}

\subsubsection{Duffing drift}
We further illustrate the previous idea by adding a drift term $G(t) = \frac{1}{10}\cos(t)$, which is derived from the Duffing oscillator. 
\begin{equation}
	mv'(t) = F(v(t)) - \int_0^t \gamma(t-s)v(s) ds + R(t) + G(t).
\end{equation}
It is straightforward to apply the previous algorithms and use equation \eqref{eq_cov_func_ensemble} to estimate the correlation functions, since $\innerp{G(t), v(0)} = 0$ as long as $\mathbb{E}v(0) = 0$. The appearance of the Duffing drift terms further impedes the estimation of the correlation functions using temporal integrals, similar to the case in the double well potential. Here, we give an asymptotic behavior of the performance of our algorithm with a changing number of data, both in time length $L$ and number of independent trajectories $M$. The experiment is performed for a batch of 10 independent trials, and the mean (solid lines) and quantiles (shaded region) are presented in Figure \ref{fig:F_deep_double_well_G_Duffing_compare}. We first fix $M = 1$ and change $L$ from $2^{11}$ to $2^{18}$. The other test fixes \reva{$L = 2^{12}$} and changes $M$ from $1$ to $2000$. We compare the performance based on the number of summands in equation \eqref{eq_cov_func_ensemble}, i.e. $M\times L$. It shows that the error decays with an increasing number of trajectories $M$ with an approximate order of $1/2$, which agrees with the convergence rate of the Monte Carlo integral, showing the well-posedness of our estimator. However, the error decays slowly with increasing trajectory length $L$. 
\begin{figure}[h]
		\centering
		\includegraphics[width=0.5\textwidth]{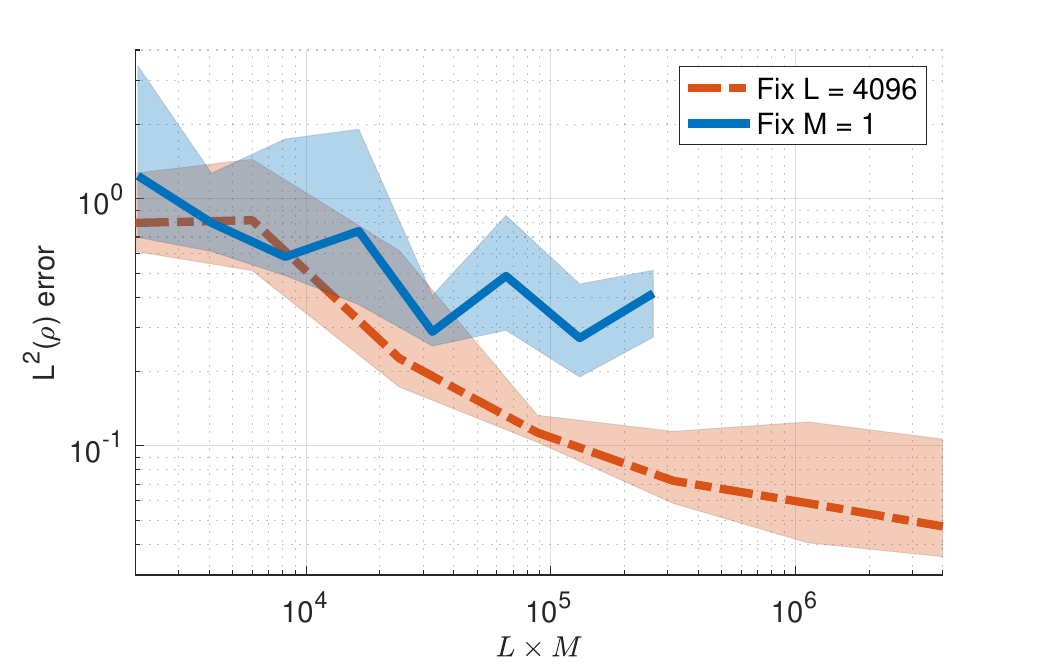}
		\caption{Compare the result with changing trajectory length $L$ and number of independent trajectories $M$, with double well external force and Duffing drift. Fix $L = 4096$ and change $M$ in $[1, 2000]$, and fix $M  = 1$, change $L$ in $[2^{11}, 2^{18}]$. The error decays with $M$ with approximate order 1/2, which agrees with the order of Monte Carlo integration and shows the well-posedness of our estimator. The error barely decays with increasing $L$. The X-axis of the two tests is matched by $L \times M$, the number of summands in the estimation of correlation functions. }
		\label{fig:F_deep_double_well_G_Duffing_compare}
\end{figure}

\section{Conclusion and future work}\label{sec_conclusion}
We propose a Sobolev norm loss function for the estimation of the memory kernel in the generalized Langevin equation. Based on this loss function, we provide a least square estimation algorithm with a performance guarantee. The accuracy of derivatives of the correlation functions used in the loss function is achieved through a regularized Prony method, assuming the exponential decay of the correlation functions. The comparison between commonly used loss functions is shown by numerical examples. The estimation of the correlation functions can be generalized to utilize multiple trajectories, with the presence of force term and drift term.

\revb{Our two-stage procedure can also be interpreted as a hierarchical regularization framework of bias-variance trade-off. The empirical correlation functions \( h(t) \) and $g(t)$ are projected onto a low-dimensional exponential basis using the Prony method and further smoothed through RKHS regularization, which significantly reduces noise but introduces bias due to model truncation and possible loss of long-time memory effects. The second stage applies additional regularization by solving a regression problem with a mixed Sobolev loss. The tuning parameters \( \omega \), \( \alpha_1 \), and \( \alpha_2 \) act as smoothing controls, as larger $\omega$ emphasizes low-frequency behavior, leading to lower variance but greater bias. }

Future work can be conducted in the following aspects.
\begin{itemize}
	\item Extend the current framework to higher dimensions, for which case the memory kernel becomes a matrix. A similar analysis should be performed for each matrix element. 
	\item Express the error of correlation functions in terms of the parameters of Prony methods, trajectory length, and observation noise. \revb{Also, express the estimation error in terms of the parameter $\omega$, $\alpha_1$ and $\alpha_2$, which explicitly shows the bias-variance trade-off. }
\end{itemize}

\section*{Acknowledgment}

This work is supported in part by the National Science Foundation via award NSF DMS-2309378. The authors would like to thank Xiantao Li for the helpful discussions and suggestions which have helped to improve the overall presentation of this paper.

\appendix

\section{Learning Power Law memory kernel}\label{apdx}
\revb{
Here we present the numerical example of learning a power-law decaying memory kernel with no external force ($F=0$), where the kernel is given by 
\begin{equation}
	\gamma(t) = \frac{1-3t^2}{(1+t^2)^3}.
\end{equation}
The experiment uses $\rho(t) = e^{-0.1t}$. The true trajectory is simulated with time step $\Delta t = 0.0125$, and observations are taken every $10 \Delta t = 0.125$ with noise level $\sigma_{obs} = 0.01$. We estimate 30 discrete values of the autocorrelation functions and apply the Prony method with 10 modes for interpolation. For regression, we use cubic splines on $[0, 30]$ with 50 knots. 

As noted in Remark \ref{rmk_powerlaw}, the kernel $\gamma$ exhibits long memory compared to exponentially decaying kernels, resulting in significantly slower convergence to stationarity. This phenomenon has been rigorously analyzed in the literature; see, for example, \cite{herzog2023gibbsian}. We illustrate this effect numerically by comparing the estimated correlation functions $h$ and $g$ under two different data regimes: a single long trajectory ($L = 2^{16}, M = 1$) and an ensemble of short trajectories ($L = 2^{12}, M = 2000$). See the top panel of Figure \ref{fig:powerlaw_all}.
\begin{figure}[thb!]
		\centering
		\includegraphics[width=0.9\textwidth]{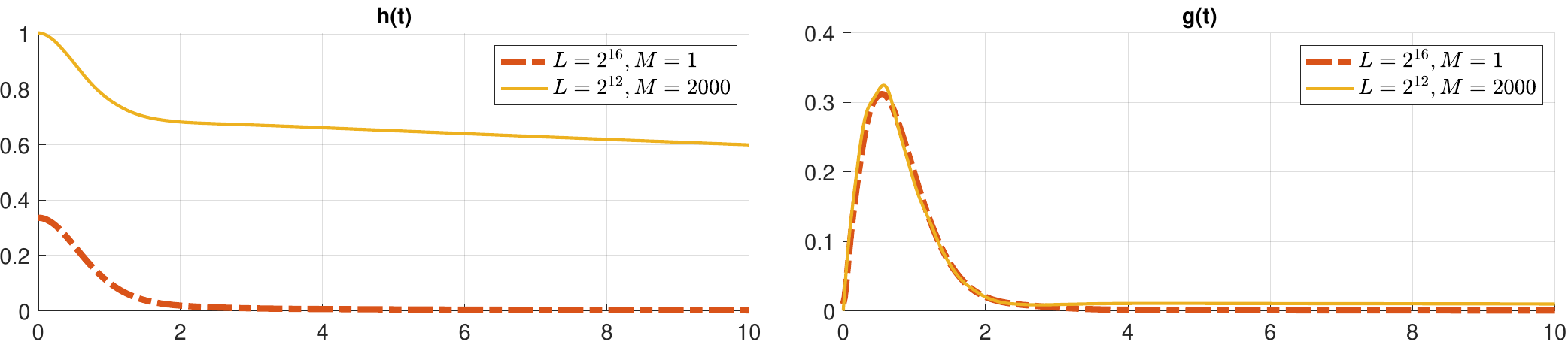}
		\vspace{1em}
		\includegraphics[width=0.9\textwidth]{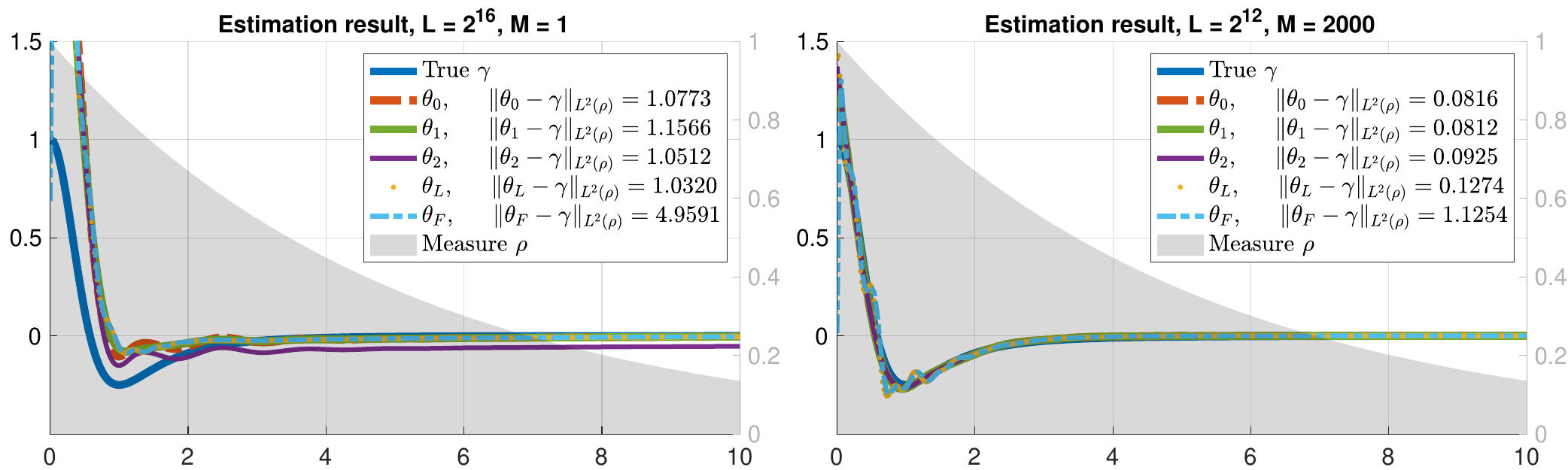}
		\caption{Power-law memory kernel. \textbf{Top}: Comparison of autocorrelation function approximations obtained via temporal averaging and ensemble (sample) averaging. \textbf{Bottom}: Corresponding memory kernel estimation result. The ensemble average yields more reliable autocorrelation estimates, resulting in more accurate memory kernel recovery.}
		\label{fig:powerlaw_all}
\end{figure}

The correlation functions $h$ and $g$ estimated from temporal averaging do not satisfy the equation $g = h * \gamma$. In contrast, we argue that the correlation functions obtained through ensemble (sample) averaging provide more accurate estimates. Consequently, regression-based estimators for $\gamma$ exhibit higher errors when relying on temporally averaged data. Moreover, due to the presence of a branch cut along the negative real axis in the Laplace transform of $\gamma$, the quantity $M_\omega^\gamma$ becomes large as $\omega \to 0$, weakening the theoretical error bounds. Also, since explicit correlation functions are not available, we did not provide the theoretical guarantee of the performance. Although higher values of $\omega$ offer improved theoretical guarantees, in practice, the estimation performance remains robust across a wide range of $\omega$. A similar phenomenon was observed in Section \ref{sec:changing_parameters}. For comparison, we also include results from Tikhonov–Fourier deconvolution as a baseline. This method suffers from the Gibbs phenomenon near zero, leading to significant estimation errors. The estimation results are also comparable to those of inverse Laplace transform methods, as both frequency-domain approaches exhibit oscillations caused by the difficulty of regularization. The estimation results are presented at the bottom of Figure \ref{fig:powerlaw_all}.

As a comparison, we also present baseline results for learning the memory kernel $\gamma(t) = e^{-t}$. Both temporal and ensemble averaging yield accurate approximations of the correlation functions, with the ensemble average offering slightly better accuracy. In this setting, the true correlation functions can be derived explicitly, allowing us to compute the $L^2(\rho)$ errors for $h$ and $g$. Combined with Theorem \ref{thm_main_error_conv}, this enables a quantitative performance guarantee. The resulting error bounds are $0.0437$ and $0.0705$ for the single long trajectory and ensemble short trajectory settings. These results are summarized in Figure \ref{fig:exp_result}.

\begin{figure}[thb!]
    \centering
    \includegraphics[width=0.9\textwidth]{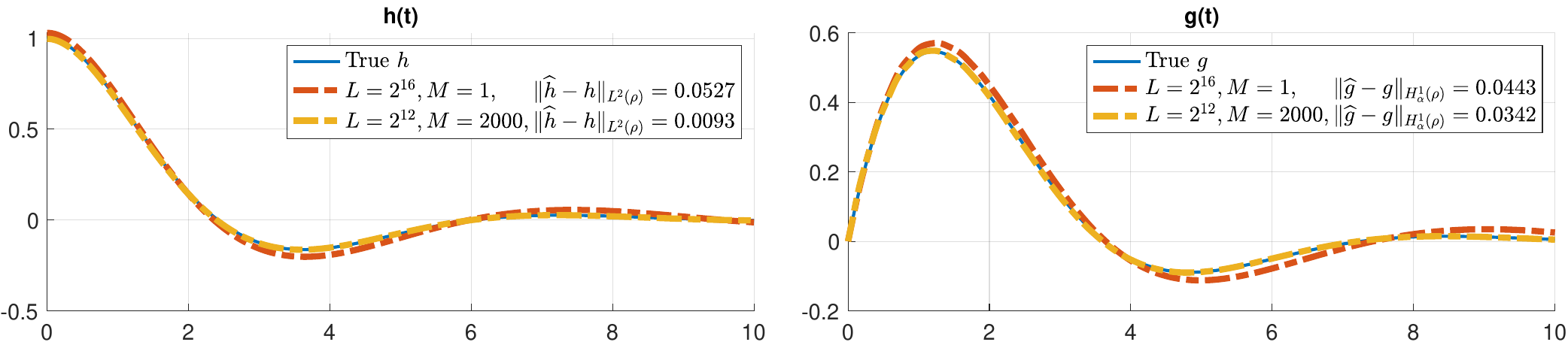}
    \vspace{1em}
    \includegraphics[width=0.9\textwidth]{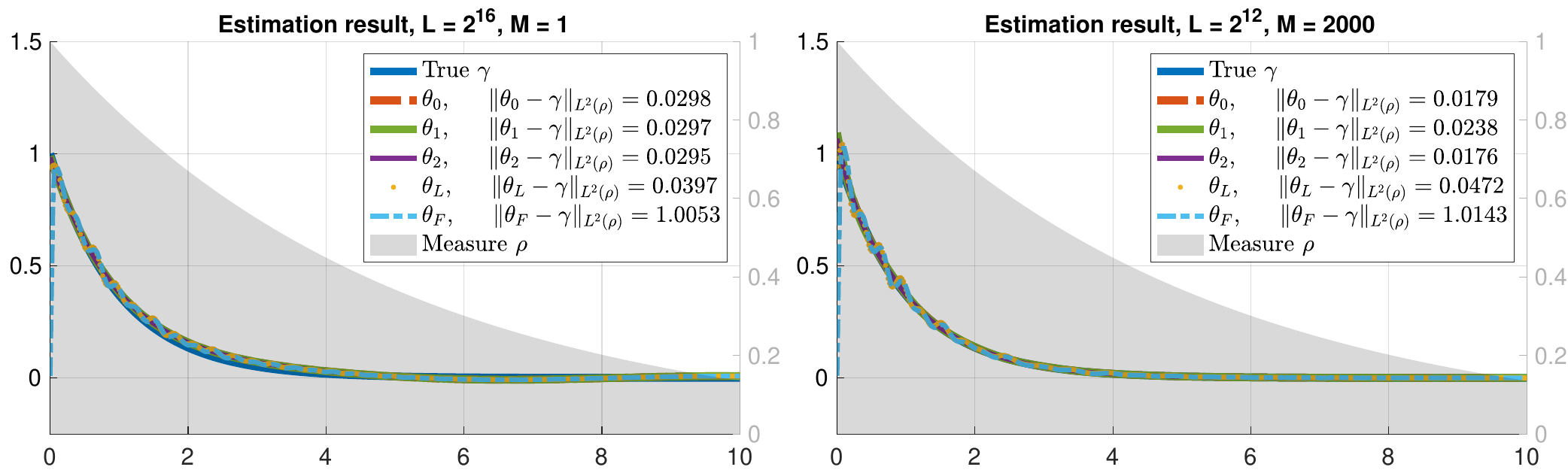}
    \caption{Exponential memory kernel. \textbf{Top}: Single long trajectory and ensemble short trajectory provide similar correlation functions estimation. \textbf{Bottom}: Ensemble averaging yields slightly more accurate kernel recovery due to improved autocorrelation estimates.}
    \label{fig:exp_result}
\end{figure}

}

\printbibliography

@article{herzog2023gibbsian,
  title={Gibbsian dynamics and the generalized Langevin equation},
  author={Herzog, David P and Mattingly, Jonathan C and Nguyen, Hung D},
  journal={Electronic Journal of Probability},
  volume={28},
  pages={1--29},
  year={2023},
  publisher={The Institute of Mathematical Statistics and the Bernoulli Society}
}

@article{baczewskiNumericalIntegrationExtended2013,
  title = {Numerical Integration of the Extended Variable Generalized {{Langevin}} Equation with a Positive {{Prony}} Representable Memory Kernel},
  author = {Baczewski, Andrew D. and Bond, Stephen D.},
  date = {2013-07-28},
  journaltitle = {The Journal of Chemical Physics},
  volume = {139},
  number = {4},
  pages = {044107},
  issn = {0021-9606, 1089-7690},
  doi = {10.1063/1.4815917},
  url = {https://pubs.aip.org/jcp/article/139/4/044107/72935/Numerical-integration-of-the-extended-variable},
  urldate = {2023-10-31},
  langid = {english}
}

@article{bockiusModelReductionTechniques2021,
  title = {Model Reduction Techniques for the Computation of Extended {{Markov}} Parameterizations for Generalized {{Langevin}} Equations},
  author = {Bockius, Niklas and Shea, Jeanine and Jung, Gerhard and Schmid, Friederike and Hanke, Martin},
  date = {2021-05},
  journaltitle = {Journal of Physics: Condensed Matter},
  shortjournal = {J. Phys.: Condens. Matter},
  volume = {33},
  number = {21},
  pages = {214003},
  publisher = {{IOP Publishing}},
  issn = {0953-8984},
  doi = {10.1088/1361-648X/abe6df},
  url = {https://dx.doi.org/10.1088/1361-648X/abe6df},
  urldate = {2023-10-31},
  langid = {english}
}

@article{carriereHighResolutionRadar1992,
  title = {High Resolution Radar Target Modeling Using a Modified {{Prony}} Estimator},
  author = {Carriere, Rob and Moses, Randolph L.},
  date = {1992-01},
  journaltitle = {IEEE Transactions on Antennas and Propagation},
  volume = {40},
  number = {1},
  pages = {13--18},
  issn = {1558-2221},
  doi = {10.1109/8.123348},
  url = {https://ieeexplore.ieee.org/document/123348},
  urldate = {2023-12-12},
  eventtitle = {{{IEEE Transactions}} on {{Antennas}} and {{Propagation}}}
}

@article{chenComputationMemoryFunctions2014,
  title = {Computation of the Memory Functions in the Generalized {{Langevin}} Models for Collective Dynamics of Macromolecules},
  author = {Chen, Minxin and Li, Xiantao and Liu, Chun},
  date = {2014-08-13},
  journaltitle = {The Journal of Chemical Physics},
  shortjournal = {The Journal of Chemical Physics},
  volume = {141},
  number = {6},
  pages = {064112},
  issn = {0021-9606},
  doi = {10.1063/1.4892412},
  url = {https://doi.org/10.1063/1.4892412},
  urldate = {2023-12-06}
}

@article{chorinProblemReductionRenormalization2006,
  title = {Problem Reduction, Renormalization, and Memory},
  author = {Chorin, Alexandre and Stinis, Panagiotis},
  date = {2006-12-31},
  journaltitle = {Communications in Applied Mathematics and Computational Science},
  shortjournal = {CAMCoS},
  volume = {1},
  number = {1},
  pages = {1--27},
  issn = {2157-5452, 1559-3940},
  doi = {10.2140/camcos.2006.1.1},
  url = {http://msp.org/camcos/2006/1-1/p01.xhtml},
  urldate = {2024-02-02},
  langid = {english}
}

@article{dammakQuantumThermalBath2009,
  title = {Quantum {{Thermal Bath}} for {{Molecular Dynamics Simulation}}},
  author = {Dammak, Hichem and Chalopin, Yann and Laroche, Marine and Hayoun, Marc and Greffet, Jean-Jacques},
  date = {2009-11-05},
  journaltitle = {Physical Review Letters},
  shortjournal = {Phys. Rev. Lett.},
  volume = {103},
  number = {19},
  pages = {190601},
  publisher = {{American Physical Society}},
  doi = {10.1103/PhysRevLett.103.190601},
  url = {https://link.aps.org/doi/10.1103/PhysRevLett.103.190601},
  urldate = {2024-01-29}
}

@article{dehoogImprovedMethodNumerical1982,
  title = {An {{Improved Method}} for {{Numerical Inversion}} of {{Laplace Transforms}}},
  author = {family=Hoog, given=Frank R., prefix=de, useprefix=true and Knight, John H. and Stokes, A. N.},
  date = {1982-09},
  journaltitle = {SIAM Journal on Scientific and Statistical Computing},
  shortjournal = {SIAM J. Sci. and Stat. Comput.},
  volume = {3},
  number = {3},
  pages = {357--366},
  publisher = {{Society for Industrial and Applied Mathematics}},
  issn = {0196-5204},
  doi = {10.1137/0903022},
  url = {https://epubs.siam.org/doi/10.1137/0903022},
  urldate = {2023-12-07}
}

@article{epsteinBadTruthLaplace2008,
  title = {The {{Bad Truth}} about {{Laplace}}'s {{Transform}}},
  author = {Epstein, Charles L. and Schotland, John},
  date = {2008},
  journaltitle = {SIAM Review},
  volume = {50},
  number = {3},
  eprint = {20454138},
  eprinttype = {jstor},
  pages = {504--520},
  publisher = {{Society for Industrial and Applied Mathematics}},
  issn = {0036-1445},
  url = {https://www.jstor.org/stable/20454138},
  urldate = {2023-12-07}
}

@article{frankignoulStochasticClimateModels1977,
  title = {Stochastic Climate Models, {{Part II Application}} to Sea-Surface Temperature Anomalies and Thermocline Variability},
  author = {Frankignoul, Claude and Hasselmann, Klaus},
  date = {1977-01-01},
  journaltitle = {Tellus},
  volume = {29},
  number = {4},
  pages = {289--305},
  publisher = {{Taylor \& Francis}},
  issn = {0040-2826},
  doi = {10.3402/tellusa.v29i4.11362},
  url = {https://doi.org/10.3402/tellusa.v29i4.11362},
  urldate = {2024-01-29}
}

@article{gordonGeneralizedLangevinModels2009,
  title = {Generalized {{Langevin}} Models of Molecular Dynamics Simulations with Applications to Ion Channels},
  author = {Gordon, Dan and Krishnamurthy, Vikram and Chung, Shin-Ho},
  date = {2009-10-02},
  journaltitle = {The Journal of Chemical Physics},
  shortjournal = {The Journal of Chemical Physics},
  volume = {131},
  number = {13},
  pages = {134102},
  issn = {0021-9606},
  doi = {10.1063/1.3233945},
  url = {https://doi.org/10.1063/1.3233945},
  urldate = {2024-01-29}
}

@article{groganDatadrivenMolecularModeling2020,
  title = {Data-Driven Molecular Modeling with the Generalized {{Langevin}} Equation},
  author = {Grogan, Francesca and Lei, Huan and Li, Xiantao and Baker, Nathan A.},
  date = {2020-10-01},
  journaltitle = {Journal of computational physics},
  shortjournal = {J Comput Phys},
  volume = {418},
  eprint = {32952214},
  eprinttype = {pmid},
  pages = {109633},
  issn = {0021-9991},
  doi = {10.1016/j.jcp.2020.109633},
  url = {https://www.ncbi.nlm.nih.gov/pmc/articles/PMC7494205/},
  urldate = {2024-02-02},
  pmcid = {PMC7494205}
}

@article{hauerInitialResultsProny1990,
  title = {Initial Results in {{Prony}} Analysis of Power System Response Signals},
  author = {Hauer, John F. and Demeure, C.J. and Scharf, Louis L.},
  date = {1990-02},
  journaltitle = {IEEE Transactions on Power Systems},
  volume = {5},
  number = {1},
  pages = {80--89},
  issn = {1558-0679},
  doi = {10.1109/59.49090},
  url = {https://ieeexplore.ieee.org/document/49090},
  urldate = {2024-01-02},
  eventtitle = {{{IEEE Transactions}} on {{Power Systems}}}
}

@article{huaMatrixPencilMethod1990,
  title = {Matrix Pencil Method for Estimating Parameters of Exponentially Damped/Undamped Sinusoids in Noise},
  author = {Hua, Yingbo and Sarkar, Tapan K.},
  date = {1990-05},
  journaltitle = {IEEE Transactions on Acoustics, Speech, and Signal Processing},
  volume = {38},
  number = {5},
  pages = {814--824},
  issn = {0096-3518},
  doi = {10.1109/29.56027},
  url = {https://ieeexplore.ieee.org/document/56027},
  urldate = {2023-10-31},
  eventtitle = {{{IEEE Transactions}} on {{Acoustics}}, {{Speech}}, and {{Signal Processing}}}
}

@article{huSimulationStochasticProcesses1997,
  title = {On the Simulation of Stochastic Processes by Spectral Representation},
  author = {Hu, Bin and Schiehlen, Werner},
  date = {1997-04},
  journaltitle = {Probabilistic Engineering Mechanics},
  shortjournal = {Probabilistic Engineering Mechanics},
  volume = {12},
  number = {2},
  pages = {105--113},
  issn = {02668920},
  doi = {10.1016/S0266-8920(96)00039-2},
  url = {https://linkinghub.elsevier.com/retrieve/pii/S0266892096000392},
  urldate = {2023-11-16},
  langid = {english}
}

@article{jungIterativeReconstructionMemory2017,
  title = {Iterative {{Reconstruction}} of {{Memory Kernels}}},
  author = {Jung, Gerhard and Hanke, Martin and Schmid, Friederike},
  date = {2017-06-13},
  journaltitle = {Journal of Chemical Theory and Computation},
  shortjournal = {J. Chem. Theory Comput.},
  volume = {13},
  number = {6},
  pages = {2481--2488},
  publisher = {{American Chemical Society}},
  issn = {1549-9618},
  doi = {10.1021/acs.jctc.7b00274},
  url = {https://doi.org/10.1021/acs.jctc.7b00274},
  urldate = {2024-02-02}
}

@article{kerrwinterDeepLearningApproach2023,
  title = {A Deep Learning Approach to the Measurement of Long-Lived Memory Kernels from Generalized {{Langevin}} Dynamics},
  author = {Kerr Winter, Max and Pihlajamaa, Ilian and Debets, Vincent E. and Janssen, Liesbeth M. C.},
  date = {2023-06-28},
  journaltitle = {The Journal of Chemical Physics},
  shortjournal = {J Chem Phys},
  volume = {158},
  number = {24},
  eprint = {37366311},
  eprinttype = {pmid},
  pages = {244115},
  issn = {1089-7690},
  doi = {10.1063/5.0149764},
  langid = {english}
}

@article{kuhlmanReviewInverseLaplace2013,
  title = {Review of Inverse {{Laplace}} Transform Algorithms for {{Laplace-space}} Numerical Approaches},
  author = {Kuhlman, Kristopher L.},
  date = {2013-06},
  journaltitle = {Numerical Algorithms},
  shortjournal = {Numer Algor},
  volume = {63},
  number = {2},
  pages = {339--355},
  issn = {1017-1398, 1572-9265},
  doi = {10.1007/s11075-012-9625-3},
  url = {http://link.springer.com/10.1007/s11075-012-9625-3},
  urldate = {2023-12-07},
  langid = {english}
}

@article{lammFutureSequentialRegularizationMethods1995,
  title = {Future-{{Sequential Regularization Methods}} for {{III-Posed Volterra Equations}}: {{Applications}} to the {{Inverse Heat Conduction Problem}}},
  shorttitle = {Future-{{Sequential Regularization Methods}} for {{III-Posed Volterra Equations}}},
  author = {Lamm, Patricia K.},
  date = {1995-10-15},
  journaltitle = {Journal of Mathematical Analysis and Applications},
  shortjournal = {Journal of Mathematical Analysis and Applications},
  volume = {195},
  number = {2},
  pages = {469--494},
  issn = {0022-247X},
  doi = {10.1006/jmaa.1995.1368},
  url = {https://www.sciencedirect.com/science/article/pii/S0022247X85713686},
  urldate = {2023-12-06}
}

@article{lammNumericalSolutionFirstKind1997,
  title = {Numerical {{Solution}} of {{First-Kind Volterra Equations}} by {{Sequential Tikhonov Regularization}}},
  author = {Lamm, Patricia K. and Eldén, Lars},
  date = {1997-08},
  journaltitle = {SIAM Journal on Numerical Analysis},
  shortjournal = {SIAM J. Numer. Anal.},
  volume = {34},
  number = {4},
  pages = {1432--1450},
  publisher = {{Society for Industrial and Applied Mathematics}},
  issn = {0036-1429},
  doi = {10.1137/S003614299528081X},
  url = {https://epubs.siam.org/doi/10.1137/S003614299528081X},
  urldate = {2023-12-06}
}

@article{lammRegularizedInversionFinitely1997,
  title = {Regularized Inversion of Finitely Smoothing {{Volterra}} Operators: Predictor - Corrector Regularization Methods},
  shorttitle = {Regularized Inversion of Finitely Smoothing {{Volterra}} Operators},
  author = {Lamm, Patricia K.},
  date = {1997-04},
  journaltitle = {Inverse Problems},
  shortjournal = {Inverse Problems},
  volume = {13},
  number = {2},
  pages = {375},
  issn = {0266-5611},
  doi = {10.1088/0266-5611/13/2/012},
  url = {https://dx.doi.org/10.1088/0266-5611/13/2/012},
  urldate = {2023-12-06},
  langid = {english}
}

@incollection{lammSurveyRegularizationMethods2000,
  title = {A {{Survey}} of {{Regularization Methods}} for {{First-Kind Volterra Equations}}},
  booktitle = {Surveys on {{Solution Methods}} for {{Inverse Problems}}},
  author = {Lamm, Patricia K.},
  editor = {Colton, David and Engl, Heinz W. and Louis, Alfred K. and McLaughlin, Joyce R. and Rundell, William},
  date = {2000},
  pages = {53--82},
  publisher = {{Springer}},
  location = {{Vienna}},
  doi = {10.1007/978-3-7091-6296-5_4},
  url = {https://doi.org/10.1007/978-3-7091-6296-5_4},
  urldate = {2023-12-06},
  isbn = {978-3-7091-6296-5},
  langid = {english}
}

@article{leiDatadrivenParameterizationGeneralized2016,
  title = {Data-Driven Parameterization of the Generalized {{Langevin}} Equation},
  author = {Lei, Huan and Baker, Nathan A. and Li, Xiantao},
  date = {2016-12-13},
  journaltitle = {Proceedings of the National Academy of Sciences},
  volume = {113},
  number = {50},
  pages = {14183--14188},
  publisher = {{Proceedings of the National Academy of Sciences}},
  doi = {10.1073/pnas.1609587113},
  url = {https://www.pnas.org/doi/10.1073/pnas.1609587113},
  urldate = {2023-10-31}
}

@article{liCoarsegrainedMolecularDynamics2010,
  title = {A Coarse-Grained Molecular Dynamics Model for Crystalline Solids},
  author = {Li, Xiantao},
  date = {2010},
  journaltitle = {International Journal for Numerical Methods in Engineering},
  volume = {83},
  number = {8-9},
  pages = {986--997},
  issn = {1097-0207},
  doi = {10.1002/nme.2892},
  url = {https://onlinelibrary.wiley.com/doi/abs/10.1002/nme.2892},
  urldate = {2024-02-02},
  langid = {english}
}

@article{luSemiclassicalGeneralizedLangevin2019,
  title = {Semi-Classical Generalized {{Langevin}} Equation for Equilibrium and Nonequilibrium Molecular Dynamics Simulation},
  author = {Lü, Jing-Tao and Hu, Bing-Zhong and Hedegård, Per and Brandbyge, Mads},
  date = {2019-02-01},
  journaltitle = {Progress in Surface Science},
  shortjournal = {Progress in Surface Science},
  volume = {94},
  number = {1},
  pages = {21--40},
  issn = {0079-6816},
  doi = {10.1016/j.progsurf.2018.07.002},
  url = {https://www.sciencedirect.com/science/article/pii/S0079681618300200},
  urldate = {2024-01-29}
}

@article{moriTransportCollectiveMotion1965,
  title = {Transport, {{Collective Motion}}, and {{Brownian Motion}}*)},
  author = {Mori, Hazime},
  date = {1965-03-01},
  journaltitle = {Progress of Theoretical Physics},
  shortjournal = {Progress of Theoretical Physics},
  volume = {33},
  number = {3},
  pages = {423--455},
  issn = {0033-068X},
  doi = {10.1143/PTP.33.423},
  url = {https://doi.org/10.1143/PTP.33.423},
  urldate = {2024-01-29}
}

@article{russoMachineLearningMemory2022,
  title = {Machine {{Learning Memory Kernels}} as {{Closure}} for {{Non-Markovian Stochastic Processes}}},
  author = {Russo, Antonio and Duran-Olivencia, Miguel A. and Kevrekidis, Ioannis G. and Kalliadasis, Serafim},
  date = {2022},
  journaltitle = {IEEE Transactions on Neural Networks and Learning Systems},
  shortjournal = {IEEE Trans. Neural Netw. Learning Syst.},
  pages = {1--13},
  issn = {2162-237X, 2162-2388},
  doi = {10.1109/TNNLS.2022.3210695},
  url = {https://ieeexplore.ieee.org/document/9947343/},
  urldate = {2024-02-02},
  langid = {english}
}

@article{shinozukaSimulationStochasticProcesses1991,
  title = {Simulation of {{Stochastic Processes}} by {{Spectral Representation}}},
  author = {Shinozuka, Masanobu and Deodatis, George},
  date = {1991-04-01},
  journaltitle = {Applied Mechanics Reviews},
  shortjournal = {Applied Mechanics Reviews},
  volume = {44},
  number = {4},
  pages = {191--204},
  issn = {0003-6900},
  doi = {10.1115/1.3119501},
  url = {https://doi.org/10.1115/1.3119501},
  urldate = {2023-11-16}
}

@article{wilkinsonEvaluationZerosIllconditioned1959,
  title = {The Evaluation of the Zeros of Ill-Conditioned Polynomials. {{Part I}}},
  author = {Wilkinson, James H.},
  date = {1959-12-01},
  journaltitle = {Numerische Mathematik},
  shortjournal = {Numer. Math.},
  volume = {1},
  number = {1},
  pages = {150--166},
  issn = {0945-3245},
  doi = {10.1007/BF01386381},
  url = {https://doi.org/10.1007/BF01386381},
  urldate = {2023-12-12},
  langid = {english}
}

@article{zwanzigMemoryEffectsIrreversible1961,
  title = {Memory {{Effects}} in {{Irreversible Thermodynamics}}},
  author = {Zwanzig, Robert},
  date = {1961-11-15},
  journaltitle = {Physical Review},
  shortjournal = {Phys. Rev.},
  volume = {124},
  number = {4},
  pages = {983--992},
  publisher = {{American Physical Society}},
  doi = {10.1103/PhysRev.124.983},
  url = {https://link.aps.org/doi/10.1103/PhysRev.124.983},
  urldate = {2024-01-29}
}

@article{linDatadrivenModelReduction2021,
  title = {Data-Driven Model Reduction, {{Wiener}} Projections, and the {{Koopman-Mori-Zwanzig}} Formalism},
  author = {Lin, Kevin K. and Lu, Fei},
  date = {2021-01-01},
  journaltitle = {Journal of Computational Physics},
  shortjournal = {Journal of Computational Physics},
  volume = {424},
  pages = {109864},
  issn = {0021-9991},
  doi = {10.1016/j.jcp.2020.109864},
  url = {https://www.sciencedirect.com/science/article/pii/S0021999120306380},
  urldate = {2024-03-26}
}

@article{luComparisonContinuousDiscretetime2016,
  title = {Comparison of Continuous and Discrete-Time Data-Based Modeling for Hypoelliptic Systems},
  author = {Lu, Fei and Lin, Kevin and Chorin, Alexandre},
  date = {2016-12-20},
  journaltitle = {Communications in Applied Mathematics and Computational Science},
  volume = {11},
  number = {2},
  pages = {187--216},
  publisher = {Mathematical Sciences Publishers},
  issn = {2157-5452},
  doi = {10.2140/camcos.2016.11.187},
  url = {https://msp.org/camcos/2016/11-2/p03.xhtml},
  urldate = {2024-03-26}
}

@online{xieInitioGeneralizedLangevin2024,
  title = {Ab {{Initio Generalized Langevin Equation}}},
  author = {Xie, Pinchen and Car, Roberto and E, Weinan},
  date = {2024-02-15},
  eprint = {2211.06558},
  eprinttype = {arxiv},
  eprintclass = {cond-mat, physics:physics},
  doi = {10.48550/arXiv.2211.06558},
  url = {http://arxiv.org/abs/2211.06558},
  urldate = {2024-03-26},
  pubstate = {preprint}
}

@inproceedings{luDataAdaptiveRKHS2022,
  title = {Data Adaptive {{RKHS Tikhonov}} Regularization for Learning Kernels in Operators},
  booktitle = {Proceedings of {{Mathematical}} and {{Scientific Machine Learning}}},
  author = {Lu, Fei and Lang, Quanjun and An, Qingci},
  date = {2022-09-14},
  pages = {158--172},
  publisher = {{PMLR}},
  issn = {2640-3498},
  url = {https://proceedings.mlr.press/v190/lu22a.html},
  urldate = {2024-01-02},
  eventtitle = {Mathematical and {{Scientific Machine Learning}}},
  langid = {english}
}

@article{razavy1962analytical,
  title={Analytical solutions for velocity-dependent nuclear potentials},
  author={Razavy, M and Field, G and Levinger, JS},
  journal={Physical Review},
  volume={125},
  number={1},
  pages={269},
  year={1962},
  publisher={APS}
}

@article{pucacco2004integrable,
  title={On integrable Hamiltonians with velocity dependent potentials},
  author={Pucacco, Giuseppe},
  journal={Celestial Mechanics and Dynamical Astronomy},
  volume={90},
  pages={109--123},
  year={2004},
  publisher={Springer}
}

@article{de2017spin,
  title={Spin-and velocity-dependent nonrelativistic potentials in modified electrodynamics},
  author={de Brito, GP and Malta, PC and Ospedal, LPR},
  journal={Physical Review D},
  volume={95},
  number={1},
  pages={016006},
  year={2017},
  publisher={APS}
}

@article{tabatabai2013novel,
  title={Novel applications of laser Doppler vibration measurements to medical imaging},
  author={Tabatabai, Habib and Oliver, David E and Rohrbaugh, John W and Papadopoulos, Christopher},
  journal={Sensing and Imaging: An International Journal},
  volume={14},
  pages={13--28},
  year={2013},
  publisher={Springer}
}

@article{atlas1973doppler,
  title={Doppler radar characteristics of precipitation at vertical incidence},
  author={Atlas, David and Srivastava, RC and Sekhon, Rajinder S},
  journal={Reviews of Geophysics},
  volume={11},
  number={1},
  pages={1--35},
  year={1973},
  publisher={Wiley Online Library}
}

\end{document}